\documentclass{article}

\PassOptionsToPackage{numbers, compress}{natbib}
 \usepackage[preprint]{neurips_2026}

\usepackage{amsmath,amsfonts,bm}

\def\eqref#1{equation~\ref{#1}}
\def\Eqref#1{Equation~\ref{#1}}

\def\1{\bm{1}}
\newcommand{\data}{\mathcal{D}}

\DeclareMathAlphabet{\mathsfit}{\encodingdefault}{\sfdefault}{m}{sl}
\SetMathAlphabet{\mathsfit}{bold}{\encodingdefault}{\sfdefault}{bx}{n}

\usepackage[utf8]{inputenc} 
\usepackage[T1]{fontenc}    
\usepackage{hyperref}       
\usepackage{url}            
\usepackage{booktabs}       
\usepackage{amsfonts}       
\usepackage{nicefrac}       
\usepackage{microtype}      
\usepackage{xcolor}         
\usepackage{amsmath,amssymb,amsthm}
\usepackage{bbm}
\usepackage{mathrsfs}
\usepackage{graphicx}
\usepackage{caption}
\usepackage{array}
\usepackage{makecell}
\usepackage{threeparttable}
\usepackage{pifont}
\usepackage{pdflscape}
\usepackage{longtable}
\usepackage{wrapfig}
\usepackage{svg}
\usepackage{enumitem}
\usepackage{subcaption}

\numberwithin{equation}{section}
\newtheorem{theorem}{Theorem}[section]
\newtheorem{proposition}{Proposition}[section]
\newtheorem{corollary}{Corollary}[section]
\newtheorem{definition}{Definition}[section]
\newtheorem{assumption}{Assumption}[section]

\definecolor{CheckGreen}{HTML}{238B45}
\definecolor{CrossRed}{HTML}{CB181D}
\newcommand{\benchcmark}{\textcolor{CheckGreen}{\ding{51}}}
\newcommand{\benchxmark}{\textcolor{CrossRed}{\ding{55}}}

\def\eg{\textit{e.g.}}

\definecolor{shiang}{rgb}{0.6,0.0,0.0}

\definecolor{vahid}{rgb}{0.0,0.0,0.6}

\definecolor{coopermj}{rgb}{0.0,0.6,0.0}

\title{SurvivalPFN: Amortizing Survival Prediction via In-Context Bayesian Inference}

\author{%
  Shi-ang Qi$^{1}$ \quad Vahid Balazadeh$^{1,2}$ \quad Michael Cooper$^{1,2}$ \\ \textbf{Russell Greiner$^{3,4}$ \quad Rahul G. Krishnan$^{1,2}$}\\
  {}$^1$ Vector Institute \quad {}$^2$ University of Toronto \quad  
  {}$^3$ University of Alberta \\ {}$^4$ Alberta Machine Intelligence Institute
}

\begin{document}

\maketitle

\addtocontents{toc}{\protect\setcounter{tocdepth}{-1}}

\begin{abstract}
    Survival analysis provides a powerful statistical framework for modeling time-to-event outcomes in the presence of censoring. 
    However, selecting an appropriate estimator from the many specialized survival approaches often requires substantial methodological and domain expertise. 
    We introduce SurvivalPFN, a prior-data fitted network that amortizes Bayesian inference for censored observations through in-context learning. 
    SurvivalPFN is pretrained on a diverse family of synthetic, identifiable, and right-censored data-generating processes, enabling it to amortize survival analysis in a single forward pass during inference. 
    As a result, the model adapts to the effective complexity of each dataset without task-specific training or hyperparameter tuning, avoids restrictive parametric assumptions, and produces calibrated survival distributions. 
    In a large-scale benchmark spanning 61 datasets, 21 methods, and 5 evaluation metrics, SurvivalPFN achieves strong predictive performance and often improves upon established survival models. 
    These results suggest that SurvivalPFN offers a principled and practical foundation model for survival analysis, with potential applications in high-impact domains such as healthcare, finance, and engineering (\url{https://github.com/rgklab/SurvivalPFN}).
\end{abstract}

\section{Introduction}

\begin{wrapfigure}{r}{9cm}
   \vspace{-27pt}
    \centering
    \includegraphics[width=\linewidth]{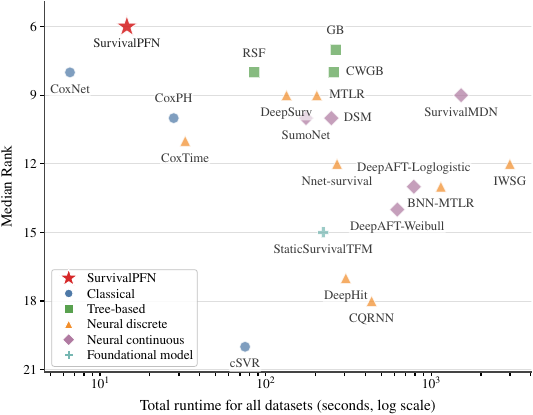}
    \caption{\textbf{Computational efficiency vs. performance across 61 datasets and 5 metrics.} SurvivalPFN achieves the best median rank while matching classical models in speed.}
    \label{fig:teaser}
    \vspace{-20pt}
\end{wrapfigure}

Survival analysis models the distribution of time to an event of interest, with applications spanning medicine~\citep{li2021predicting, she2020development, qi2022personalized, cooper2023machine, cooper2024dynameld, gao2024predicting}, e-commerce~\citep{lu2002predicting, perianez2016churn, equihua2023modelling}, engineering~\citep{papathanasiou2023machine, broadway2001survival, kuhajda2016using}, and finance~\citep{luoma1991survival, gepp2008role, giot2007ipos}. Such models are learned and evaluated on data that often exhibits \textit{right-censoring}: for some instances, the event is not observed during the follow-up period, so we only know that the event time exceeds the censoring time.

Various survival analysis methods have been proposed to handle right-censored data, but each imposes different inductive biases. 
Classical models such as Cox proportional hazards (CoxPH)~\citep{cox1972regression} often rely on constant hazard ratios and linear covariate effects.
Ensemble methods and deep survival models improve flexibility, but typically require careful tuning and often retain structural assumptions through parametric forms~\citep{ranganath2016deep, norman2024deepaft}, proportional hazards~\citep{katzman2018deepsurv}, fixed time/quantile discretizations~\citep{yu2011learning,lee2018deephit,pearce2022censored}, or mixture-based continuous-time distributions~\citep{nagpal2021dsm,han2022survival}. 
Consequently, practitioners must navigate a large set of estimators with distinct assumptions and limitations; model selection, training, and validation require substantial domain and methodological expertise.

This work aims to design a survival estimator that \textit{(i) avoids rigid simplifying assumptions}; \textit{(ii) adapts to the effective complexity of the observed data}; and \textit{(iii) enables efficient inference without extensive training or hyperparameter tuning}.

\begin{figure}
    \centering
    \includegraphics[width=\textwidth]{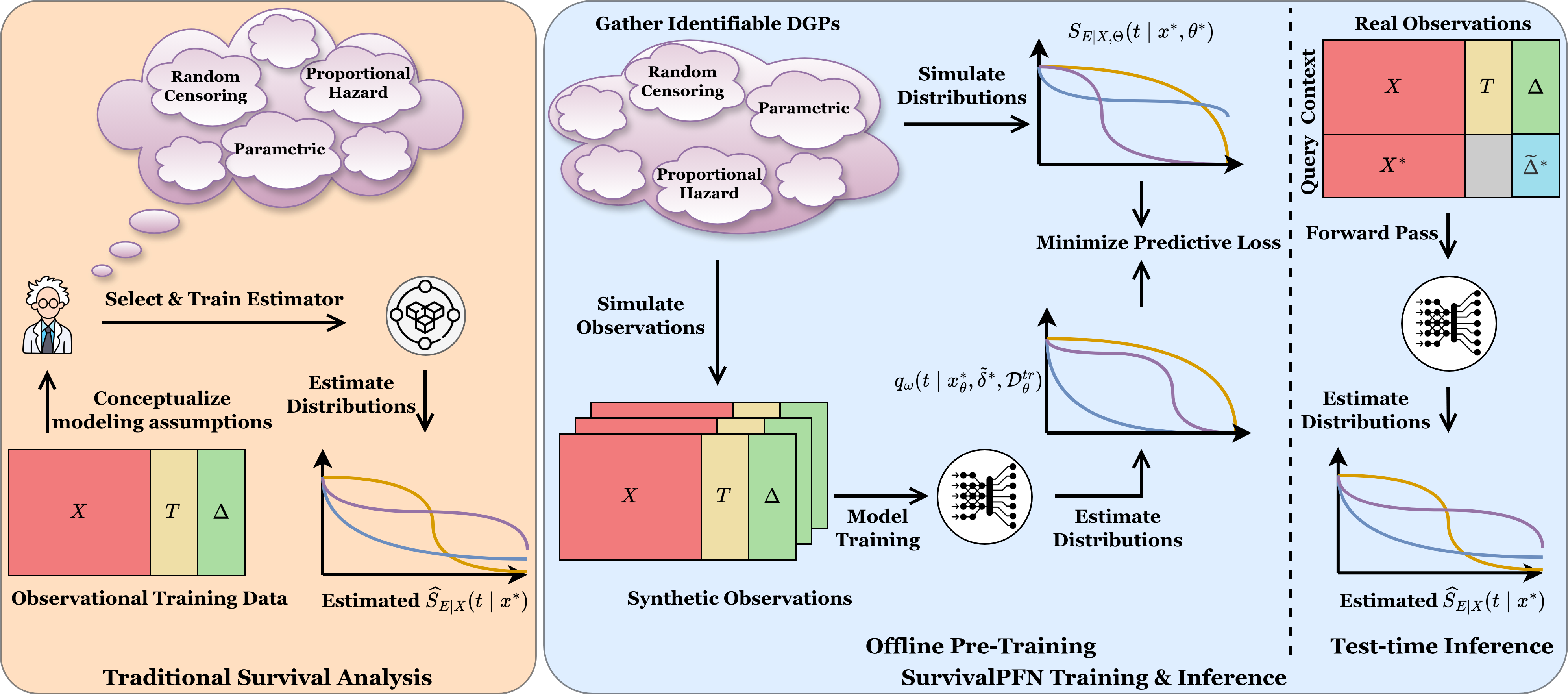}
    \caption{\textbf{Traditional survival analysis vs. SurvivalPFN.} \textit{(Left):} Traditional survival analysis requires an analyst to select and fit a suitable estimator for the observational data. \textit{(Right):} SurvivalPFN pre-trains on diverse synthetic, identifiable DGPs. At inference, an observed dataset is provided as context, and the survival distributions for query instances are obtained with a single forward pass.}
    \label{fig:overview}
\end{figure}

To do so, we build on prior-data fitted networks (PFNs)~\citep{muller2021transformers}: transformer-based models~\citep{vaswani2017attention} that learn in-context approximations of posterior predictive distributions using synthetic tasks. 
Rather than fitting a new survival model for each dataset, SurvivalPFN shifts computation to an offline prior-data pretraining stage. 
At inference time, an observed right-censored dataset is provided as context, and a single forward pass returns posterior survival distributions for new individuals. 
This approach provides a practical route to Bayesian survival prediction that avoids dataset-specific optimization and extensive hyperparameter tuning.

We present \textit{SurvivalPFN}, a transformer model for survival prediction via in-context learning. 
Our framework uses a general-purpose prior under conditional independent censoring to generate millions of simulated data generating processes (DGPs). 
By training on these diverse DGPs, SurvivalPFN learns to infer conditional survival distributions directly from observed right-censoring data, yielding an easy-to-use and efficient estimator with strong empirical performance; see Figure~\ref{fig:teaser}. 
Figure~\ref{fig:overview} contrasts the SurvivalPFN workflow with traditional survival modeling.
Our key contributions:
\begin{enumerate}
    \item We introduce a framework for amortized Bayesian survival prediction via large-scale pretraining. SurvivalPFN uses a single forward pass to adapt to data complexity without task-specific training or hyperparameter tuning, while avoiding restrictive parametric assumptions.

    \item We provide theoretical justification for SurvivalPFN as an asymptotically consistent estimator under identifiable right-censored data-generating processes.

    \item We conduct a large-scale benchmark comparing $21$ models across $61$ datasets and $5$ evaluation metrics, making this, to our knowledge, one of the largest survival model benchmarking studies to date. SurvivalPFN achieves the best median rank.
    
    \item We release the code for training SurvivalPFN, together with a \texttt{scikit-learn}-style API (see Supplementary Material).
\end{enumerate}

\section{Background}
\label{sec:background}

\textbf{Survival Analysis and Prediction.}
Let $X \in \mathbb{R}^d$ denote a covariate vector, and $E, C \in \mathbb{R}_+$ represent the event and censoring times. 
We assume a true joint distribution $P$ over the tuple $(X, E, C, T, \Delta)$, where $T = \min(E, C)$ and $\Delta = \mathbbm{1}[E\leq C]$. Specifically, we observe $N$ draws $\data=\{(x_i,t_i,\delta_i)\}_{i=1}^N$ from the distribution on observed variables $P_{\mathrm{obs}}(X, T, \Delta)$; $E$ and $C$ are latent variables; only $T$ and $\Delta$ are observed. Survival predictors aim to learn the conditional densities or survival functions:
\begin{align*}
    f_{E\mid X}(t\mid x) \ = \ \Pr(E=t \mid X=x), \quad \text{or (equivalently)} \quad S_{E\mid X}(t\mid x) \ = \ \Pr(E>t\mid X=x).
\end{align*}

\textbf{Identifiable Survival Analysis.}
We say that the
conditional survival function $S_{E\mid X}$ is (nonparametrically) \emph{identified} if any two
candidate data generating processes that induce the same observed law over $(X,T,\Delta)$ must also induce the same event-time survival function~\citep{tsiatis1975nonidentifiability, tsiatis2006semiparametric}:
\begin{align*}
    P^{(1)}_{\mathrm{obs}}(X,T,\Delta) \ = \ P^{(2)}_{\mathrm{obs}}(X,T,\Delta)
    \qquad\Longrightarrow\qquad
    S^{(1)}_{E\mid X}(t\mid x) \ = \ S^{(2)}_{E\mid X}(t\mid x),
\end{align*}
for almost every $x$ and all $t$ in the identifiable support.

A sufficient condition for nonparametric identification is given by
the following standard assumptions.
\begin{assumption}[Conditional independent censoring]
$E \perp C \mid X$.
\end{assumption}
\begin{assumption}[Positivity]
For a time region $\mathcal{T}$ of interest,
$\Pr(C \ge t \mid X=x)>0$, $\forall t\in\mathcal{T}$.
\end{assumption}

When $E\not\perp C\mid X$, the event distribution is generally not
nonparametrically identifiable from $(X,T,\Delta)$ alone; identification then
requires additional assumptions, \eg, a specified copula family
for the dependence between $E$ and $C$
\citep{gharari2023copula,zhang2024deep}.
Appendix~\ref{app:identifiable} includes more theory on identifiability.

\textbf{Bayesian Survival Prediction.}
Consider a family of identifiable survival data-generating processes indexed by $\theta\in\Theta$, with prior $\pi(\cdot)$ over $\Theta$. 
Each $\theta$ induces conditional densities and survival functions for both event and censoring times, $f_{E\mid X,\Theta}(e\mid x,\theta)$ and $S_{E\mid X,\Theta}(e\mid x,\theta)$.
In Bayesian survival modeling, we place a prior density $f_{\Theta}(\theta)$ and infer the posterior density via Bayes' rule,
\begin{equation}
    f_{\Theta\mid \mathscr{D}}(\theta \mid \data) \ \propto \ f_{\mathscr{D}\mid \Theta}(\data\mid \theta) f_{\Theta}(\theta).
    \label{eq:survival-bayes}
\end{equation}
Under conditional independent censoring, the likelihood can be decomposed into:
\begin{equation*}
\resizebox{\linewidth}{!}{$\displaystyle
    f_{\mathscr{D}\mid \Theta}(\data\mid \theta) =
    \prod_{i=1}^N
    \left[
        f_{E\mid X,\Theta}(t_i\mid x_i,\theta) \,
        S_{C\mid X,\Theta}(t_i\mid x_i,\theta)
    \right]^{\delta_i}
    \left[
        f_{C\mid X,\Theta}(t_i\mid x_i,\theta) \,
        S_{E\mid X,\Theta}(t_i\mid x_i,\theta)
    \right]^{1-\delta_i}
$}
\end{equation*}
where, for $A\in\{E,C\}$, $S_{A\mid X,\Theta}(t_i\mid x_i,\theta) = \int_{t_i}^{\infty} f_{A\mid X,\Theta}(\tau\mid x_i,\theta)\,d\tau$.
Given a new covariate vector $x^\ast$, the Bayesian posterior predictive distribution (PPD) of the event time is
\begin{equation}
    f_{E\mid X,\mathscr{D}}(t\mid x^\ast,\data) \ = \
    \int_{\Theta}
    f_{E\mid X,\Theta}(t\mid x^\ast,\vartheta)
    f_{\Theta\mid \mathscr{D}}(\vartheta\mid \data)\,d\vartheta .
    \label{eq:PPD}
\end{equation}
Analogously, the posterior predictive survival distribution (PPSD) is
\begin{equation}
    S_{E\mid X,\mathscr{D}}(t\mid x^\ast,\data) \ = \ \int_{\Theta}
    S_{E\mid X,\Theta}(t\mid x^\ast,\vartheta)
    f_{\Theta\mid \mathscr{D}}(\vartheta\mid \data)\,d\vartheta .
    \label{eq:PPSD}
\end{equation}
This framework is attractive because it integrates over plausible survival mechanisms rather than committing to a single fitted model. 
However, it is difficult to use directly with flexible survival models: evaluating the likelihood can require numerical integration to obtain $S_{E\mid X,\Theta}$; the normalizing constant in \Eqref{eq:survival-bayes} and the posterior predictive integrals in Equations~\ref{eq:PPD} and \ref{eq:PPSD} are generally intractable; and approximate inference methods such as Markov chain Monte Carlo (MCMC)~\citep{neal2012bayesian,andrieu2003introduction,welling2011bayesian} or variational inference (VI)~\citep{jordan1999introduction,wainwright2008graphical,hoffman2013stochastic,qi2023effective} must be rerun for each new dataset. 
We therefore seek an amortized procedure that preserves the posterior-predictive interpretation of Bayesian survival prediction while avoiding dataset-specific posterior computation.

\textbf{Prior-Data Fitted Networks and Amortized Bayesian Inference.} 
Prior-data fitted networks (PFNs) amortize Bayesian posterior prediction by training transformers on synthetic tasks sampled from a prior-data generating process~\citep{muller2021transformers,muller2025position}. 
Each task consists of a context set and query inputs, and the PFN is trained to predict query targets according to the posterior predictive distribution induced by the prior. 
After pretraining, posterior inference is no longer explicit: the transformer's in-context computation maps a new dataset and query points directly to predictive distributions in a single forward pass, replacing dataset-specific MCMC or VI. 
This connects PFNs to meta-learning~\citep{finn2017model}, but replaces task-specific adaptation with in-context inference. 
PFNs have achieved strong transfer in tabular prediction~\citep{hollmann2022tabpfn,hollmann2025accurate,qu2026tabiclv2}, causal effect estimation~\citep{balazadeh2025causalpfn,robertson2025pfn}, and time-series prediction~\citep{hoo2025tables,berghaus2025context}, motivating our use of this paradigm for amortized Bayesian survival prediction.

\section{Method}
\label{sec:method}



\subsection{SurvivalPFN: Amortized Posterior Predictive Inference}
\label{sec:survivalpfn}

\textbf{Overview.} SurvivalPFN learns an in-context approximation to Bayesian posterior predictive inference for right-censored survival data.
Instead of specifying a tractable likelihood and performing posterior inference separately for each dataset, we specify a prior through a simulator over identifiable right-censored DGPs. 
A draw $\theta \sim \pi(\cdot)$ determines a joint law $P^\theta$ over $(X,E,C,T,\Delta)$. 
For each synthetic task, the simulator produces an observed right-censored context dataset $\data^{tr}_{\theta}=\{(x_i,t_i,\delta_i)_{\theta}\}_{i=1}^{N}$, together with held-out query covariates $x^\ast_\theta$ and their latent event and censoring times $(e^\ast_\theta,c^\ast_\theta)$. 
The latent times are used only during prior-data training; at inference time, SurvivalPFN receives the same information available in ordinary survival prediction: an observed dataset $\data$ and query covariates $x^\ast$.

\textbf{Architecture.}
Let $q_\omega$ denote a transformer with parameters $\omega$. 
Given a context dataset and a query covariate, SurvivalPFN outputs a predictive distribution over time. 
We additionally provide a binary \emph{query indicator} $\widetilde{\delta}^{\ast}$, which is a control input specifying which PPD the model should return:
\begin{align}
    q_{\omega}\!\left(t \,\middle\vert\, x^\ast_{\theta},\widetilde{\delta}^{\ast},\data^{tr}_{\theta}\right)
    \ \approx \ 
    \begin{cases}
    f_{E\mid X,\data}\!\left(t\mid x^\ast_{\theta},\data^{tr}_{\theta}\right),
    & \quad \widetilde{\delta}^{\ast}=1,\\[2mm]
    f_{C\mid X,\data}\!\left(t\mid x^\ast_{\theta},\data^{tr}_{\theta}\right),
    & \quad \widetilde{\delta}^{\ast}=0.
    \end{cases}
    \label{eq:survivalpfn_ppd_target}
\end{align}
Thus, $\widetilde{\delta}^{\ast}=1$ asks the model for the event-time PPD, whose tail probability gives the PPSD, while $\widetilde{\delta}^{\ast}=0$ asks for the posterior predictive censoring distribution (PPCD). 
This indicator is not an observed event label for the query point; rather, it specifies the prediction target.

\begin{figure}[t]
    \centering
    \includegraphics[width=0.85\linewidth]{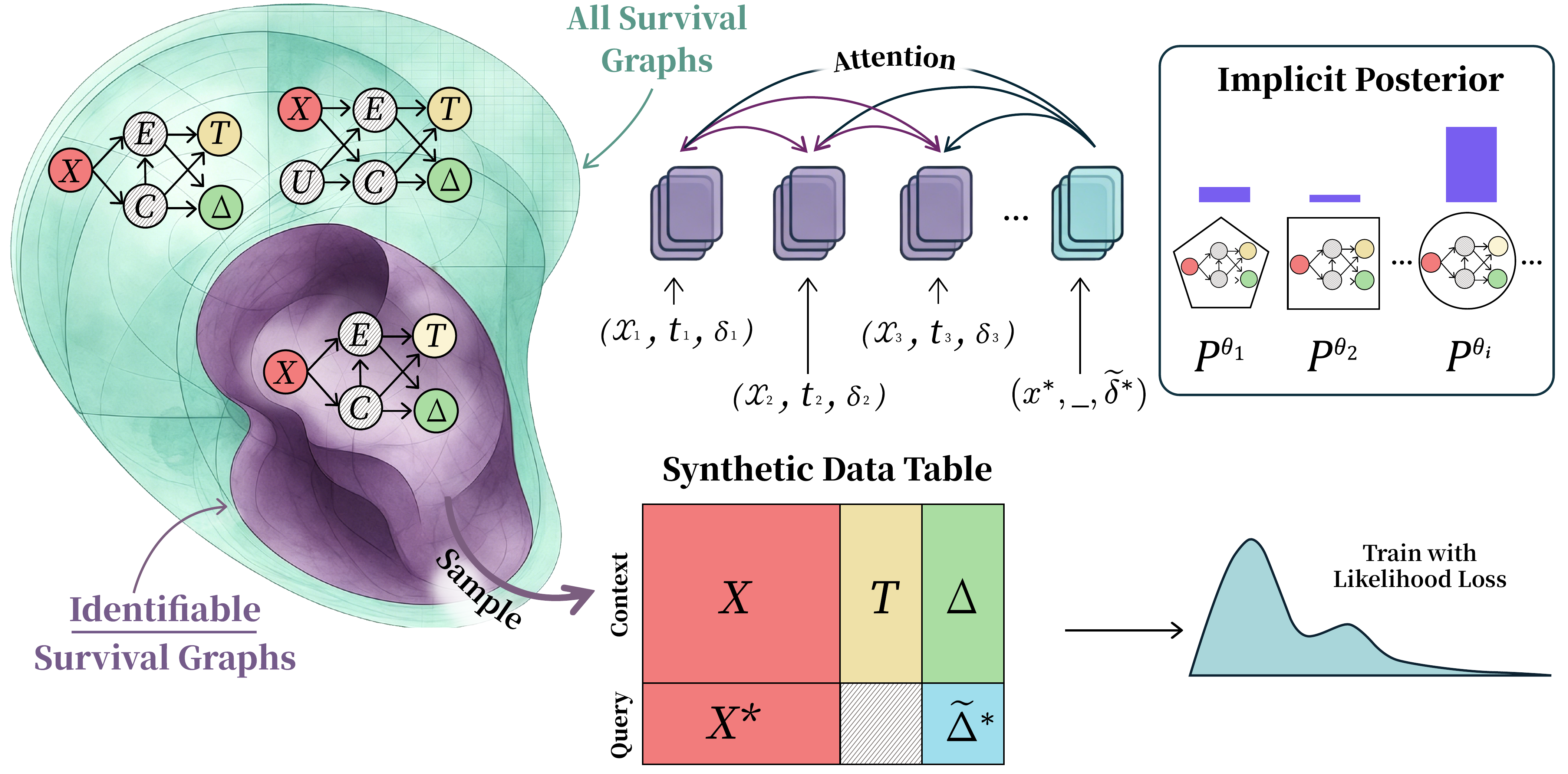}
    \caption{\textbf{Training SurvivalPFN.} At each iteration, we sample an identifiable survival DGP and use it to generate context tokens $(X,T,\Delta)$ together with query covariates $X^\ast$. Query tokens are formed by pairing $X^\ast$ with query indicators $\widetilde{\Delta}^\ast$, and SurvivalPFN predicts the requested event- or censoring-time distribution. The model is trained by minimizing the likelihood loss.}
    \label{fig:survivalpfn_training}
\end{figure}

SurvivalPFN parameterizes $q_\omega$ using the PFN-style transformer architecture of TabDPT~\citep{ma2024tabdpt} and CausalPFN~\citep{balazadeh2025causalpfn}; see Appendix~\ref{app:model_details} for details. 
As shown in Figure~\ref{fig:survivalpfn_training}, each context row $(x_i,t_i,\delta_i)_\theta \in \data^{tr}_{\theta}$ is embedded as a context token, while each query token is formed from $(x^\ast_\theta,\widetilde{\delta}^{\ast})$. 
We use three query-indicator schedules during training:
\begin{itemize}[leftmargin=10pt]
    \item \textit{Event-only:} always sets $\widetilde{\delta}^{\ast}=1$ and trains the model directly for PPSD prediction; 
    \item \textit{Both:} duplicates each query with $\widetilde{\delta}^{\ast}\in\{0,1\}$ and trains both event- and censoring-time prediction; 
    \item \textit{Random:} samples the query indicator according to the empirical censoring pattern in the context. 
\end{itemize} 
The transformer uses an asymmetric attention mask: context tokens attend to one another, while query tokens attend only to the context tokens and not to other queries, as shown by the two-way and one-way arrows in Figure~\ref{fig:survivalpfn_training}. 
Together with the absence of positional encodings, this makes predictions invariant to the ordering of the context dataset and conditionally independent across query points given the context.

The model represents each PPD as a discretized histogram over $L=1024$ time bins. 
For each query token, the transformer output is projected to logits over bins, followed by a softmax:
\begin{align*}
    q_{\omega}\!\left(\cdot \,\middle\vert\, x^\ast_\theta,\widetilde{\delta}^{\ast},\data^{tr}_{\theta}\right)
    \ = \
    \left[
    q_{\omega,\ell}(x^\ast_\theta,\widetilde{\delta}^{\ast},\data^{tr}_{\theta})
    \right]_{\ell=1}^{L}.
\end{align*}
Before discretization, we apply a monotone transformation of time, such as \texttt{lognormal2normal} or \texttt{time2quantile}, to respect the nonnegative support of survival times and allocate resolution more evenly across the context time range; see Appendix~\ref{app:monotone_transformation}. 
The predicted PPSD is obtained by summing the event-time tail probability mass:
\begin{equation}
    \widehat{S}_{\omega}(\tau_k\mid x^\ast,\data)
    \ = \
    \sum_{\ell=k+1}^{L}
    q_{\omega,\ell}(x^\ast,\widetilde{\delta}^{\ast}=1,\data).
    \label{eq:survival-pfn-survival}
\end{equation}

\textbf{Training.} Training follows the PFN principle~\citep{muller2021transformers,hollmann2022tabpfn}. 
At each gradient update, we sample $\theta\sim\pi(\cdot)$, generate a context dataset and query points from the corresponding DGP, and use the simulator-provided latent times as supervision. 
Given the query indicator, the supervised target is 
\begin{align*}
    r^\ast_\theta(\widetilde{\delta}^{\ast}) \ = \ \widetilde{\delta}^{\ast} e^\ast_\theta + \bigl(1-\widetilde{\delta}^{\ast}\bigr)c^\ast_\theta .
\end{align*}
Let $g_{\data^{tr}_{\theta}}$ be the context-fitted monotone time transformation, and let
$\kappa_{\data^{tr}_{\theta}}(r)\in\{1,\ldots,L\}$ denote the transformed-time bin containing
$g_{\data^{tr}_{\theta}}(r)$. 
We train SurvivalPFN with the discrete negative log-likelihood,
\begin{align}
\label{eq:survivalpfn_nll_loss}
    \mathcal{L}_{\mathrm{NLL}}(\omega)
    \ = \
    \mathbb{E}_{\theta\sim\pi(\cdot)}
    \, \mathbb{E}_{\data^{tr}_{\theta},\,x^\ast_\theta,\,e^\ast_\theta,\,c^\ast_\theta,\,\widetilde{\delta}^{\ast}}
    \left[
    -
    \log
    q_{\omega,\kappa_{\data^{tr}_{\theta}}\!\left(r^\ast_\theta(\widetilde{\delta}^{\ast})\right)}
    \!\left(
        x^\ast_\theta,
        \widetilde{\delta}^{\ast},
        \data^{tr}_{\theta}
    \right)
    \right].
\end{align}
This objective is a tractable prior-data likelihood for the requested latent query time. 
At the population optimum, the Bayes-optimal predictor is the conditional distribution of the latent target bin given the observed context and query. 
Thus, when the prior is restricted to identifiable right-censored DGPs, minimizing \Eqref{eq:survivalpfn_nll_loss} trains SurvivalPFN to approximate the Bayesian PPD induced by $\pi(\cdot)$. 
We also consider a smoothed cross-entropy variant, which replaces the one-hot target with a narrow Gaussian-smoothed histogram over nearby time bins; see Appendix~\ref{app:train_objective}.

\textbf{Inference.}
At inference time, SurvivalPFN is applied directly to a real right-censored dataset $\data$ and query covariates $x^\ast$. 
A single forward pass with $\widetilde{\delta}^{\ast}=1$ returns the event-time predictive distribution, from which we compute $\widehat{S}_{\omega}(t\mid x^\ast,\data)$ via \Eqref{eq:survival-pfn-survival}. 
No dataset-specific gradient updates, posterior sampling, variational optimization, or hyperparameter tuning are required. 
In this way, pretraining shifts the computational burden from test-time Bayesian inference to an offline prior-fitting stage, yielding a reusable amortized Bayesian survival predictor.

\subsection{Consistency of the Bayesian Posterior Predictive Target}
\label{sec:consistency}
We next give an informal consistency statement for the Bayesian target learned by SurvivalPFN; 
a formal version and proof are deferred to Appendix~\ref{app:survivalpfn_consistency}. 
The key point is that SurvivalPFN is consistent for conditional event distributions that are identifiable from the observed right-censored data.
\begin{proposition}[Informal consistency]
\label{prop:informal-survival-consistency}
Let $\theta^\ast$ denote the parameter of the true survival data-generating process.
Assume the prior over survival DGPs is supported only on identifiable right-censored mechanisms. 
Then, for every time $t$ in the support and for almost every query covariate $x^\ast$, the Bayesian PPSD converges almost surely to the true conditional survival function:
\begin{align*}
    S_{E\mid X,\mathscr{D}}(t\mid x^\ast,\data) \ = \
    \int_{\Theta}
    S_{E\mid X,\Theta}(t\mid x^\ast,\theta)
    f_{\Theta\mid\mathscr{D}}(\theta\mid \data)
    \,d\theta
    \quad 
    \xrightarrow[N\to\infty]{\mathrm{a.s.}}
    \quad
    S_{E\mid X,\Theta}(t\mid x^\ast,\theta^\ast).
\end{align*}
\end{proposition}

\begin{proof}[Proof sketch]
Group survival DGPs into observational equivalence classes:
\begin{align*}
    \theta_1 \ \sim\ \theta_2
    \quad\Longleftrightarrow\quad
    P^{\theta_1}_{\mathrm{obs}}(X,T,\Delta) \ = \
    P^{\theta_2}_{\mathrm{obs}}(X,T,\Delta).
\end{align*}
The observed data can distinguish different equivalence classes, but cannot distinguish DGPs within the same class. 
By applying Bayesian consistency to this quotient space, the posterior concentrates on the true observational equivalence class $[\theta^\ast]$ as $N\to\infty$. 
Therefore, the Bayesian PPSD converges to the posterior
average of $S_{E\mid X,\Theta}(t\mid x^\ast,\theta)$ over DGPs that are observationally equivalent to $\theta^\ast$, that is, over all $\theta$ satisfying $[\theta]=[\theta^\ast]$. 
If the prior is survival-identifiable, then every DGP in this equivalence class induces the same conditional event-time survival function. 
Hence this posterior average is equal to the true survival function $S_{E\mid X,\Theta}(t\mid x^\ast,\theta^\ast)$, which gives the desired consistency.
\end{proof}
Moreover, if SurvivalPFN exactly amortizes this posterior predictive
distribution, then its predicted survival curve inherits the same consistency guarantee:
\begin{align*}
    \widehat{S}_{\omega}(t\mid x^\ast,\data)  
    \xrightarrow[N\to\infty]{\mathrm{a.s.}}  
    S_{E\mid X,\Theta}(t\mid x^\ast,\theta^\ast).
\end{align*}

\subsection{Prior over Identifiable Survival DGPs}
\label{sec:survival_prior}

Optimizing our loss function does not require an explicit parameterization of the distribution $P^{\theta}$. 
Instead, it requires a prior $\pi(\cdot)$ that can generate observational datasets, consisting of tuples $(x_i, t_i, \delta_i)$, along with query points $x^\ast_\theta$ with their event/censoring times. 
Still, not every choice of prior has the desired properties. 
Proposition~\ref{prop:informal-survival-consistency} highlights two key design principles for the prior: \emph{(i)} it should rule out non-identifiable right-censored mechanisms; and \emph{(ii)} subject to this identifiability constraint, the prior should be broad enough to cover a diverse family of plausible survival mechanisms, increasing the chance that the true DGP $\theta^\ast$ lies within, or close to, the prior support.

Our prior generation relies on synthetic random multilayer perceptrons (MLPs). Following \citet{hollmann2022tabpfn}, we sample a random MLP with various numbers of layers, hidden dimensions, activations, and random initialization. Specifically, we apply additive Gaussian noise after each layer in the MLP to induce randomness in the generated data. 
To sample a synthetic right-censored dataset from our prior, we first sample a random MLP, apply it to standard Gaussian noise, and generate covariates of varying dimensions as different neurons in the MLP. Introducing randomness in the design of the MLPs aims to make the generated data more diverse. Next, we then sample two additional random MLPs, and then apply them on the generated covariates to sample event/censoring times. We then shift the sampled times to ensure non-negativity. By design, the event and censoring times are conditionally independent given the covariates, satisfying the first condition; see Figure~\ref{fig:prior_main} (left).

\begin{wrapfigure}{r}{9.2cm}
   \vspace{-10pt}
    \centering
    \includegraphics[width=\linewidth]{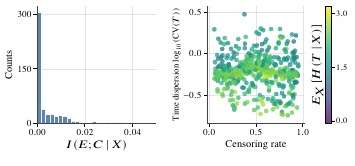}
     \vspace{-10pt}
    \caption{\textbf{Summary over 500 generated datasets.} \textit{(Left):} Histogram of conditional mutual information. \textit{(Right):} Diversity coverage over censoring rate and observed-time dispersion, colored by conditional observed-time entropy.}
    \label{fig:prior_main}
\end{wrapfigure}
To cover even more diverse survival regimes (see data diversity in Figure~\ref{fig:prior_main} right), the prior mixes four families of synthetic survival generators. 
The \emph{naive prior} treats generated table outputs as raw event and censoring times. 
The \emph{survival-distribution prior} samples smooth random monotone maps, such as Bernstein maps, and pushes uniform noise through them to obtain flexible distributions for event and censoring. 
The \emph{mixture prior} samples event and censoring times from mixtures of Weibull or log-normal components with covariate-dependent mixture parameters. 
Finally, the \emph{kitchen-sink prior} is a meta-prior that samples one of the above generators for each task, thereby increasing prior diversity.

For each synthetic dataset, we also sample one of four censoring mechanisms: \emph{uniform censoring}, where $C$ is sampled from a uniform time range; \emph{random censoring}, where $C$ is generated from an independent tabular mechanism; \emph{administrative censoring}, where subjects have different entry times but share a common study end date; and \emph{conditional independent censoring}, where both $E$ and $C$ depend on $X$ but are generated from independent conditional blocks. The observed survival data are then formed using the standard survival law.
Appendix~\ref{app:prior_generation} provides full prior-generation details.

\section{Experiments and Results}
We conduct extensive experiments to investigate the following research questions (RQs):
\begin{itemize}[leftmargin=30pt]
    \item[] \textbf{RQ1.} How does SurvivalPFN compare with survival baselines in predictive performance?
    \item[] \textbf{RQ2.} How efficient is SurvivalPFN compared with other survival estimators?
    \item[] \textbf{RQ3.} How sensitive is SurvivalPFN to the proportion of training/context samples?
    \item[] \textbf{RQ4.} How does SurvivalPFN compare with general tabular foundation models (TFMs)?
    \item[] \textbf{RQ5.} How do priors, query schedules, transformations, and losses affect performance?
\end{itemize}

\begin{wrapfigure}{r}{7.2cm}
   \vspace{-22pt}
    \centering
    \includegraphics[width=\linewidth]{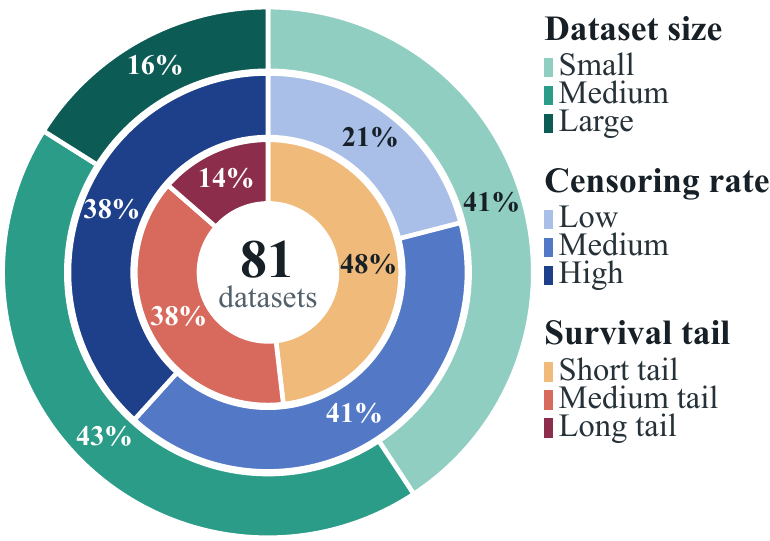}
    \vspace{-16pt}
    \caption{Dataset size (cutoffs at 500 and 5000), censoring rate and tail rate (cutoffs at 33\% and 67\%).}
    \label{fig:data_statistics}
    \vspace{-20pt}
\end{wrapfigure}
We evaluate SurvivalPFN on a large-scale benchmark covering diverse real-world regimes. 
The benchmark contains 81 datasets (see Figure~\ref{fig:data_statistics} for a preview): 20 are used only for SurvivalPFN checkpoint selection, and the remaining 61 are held out for final evaluation. 
Datasets are drawn from \texttt{SurvSet}~\citep{drysdale2022survset} and additional textbook, software-package, and recent-publication sources. 
Table~\ref{tab:dataset_summary} and Appendix~\ref{app:data} provide full dataset descriptions and summaries.

\textbf{Models.}
We evaluate 21 survival models from five families (Table~\ref{tab:survival_model_comparison}); details are in Appendices~\ref{app:model}-\ref{app:hparam_tuning}.
\begin{itemize}
    \item \textit{Tabular foundation models:} SurvivalPFN and StaticSurvivalTFM~\citep{kim2026tabular}.
    \item \textit{Classical survival models:} CoxPH~\citep{cox1972regression}, CoxNet~\citep{simon2011regularization}, cSVR~\citep{polsterl2015fast}. 
    \item \textit{Tree-based models:} GB~\citep{ridgeway1999state}, CWGB~\citep{hothorn2006survival}, and RSF~\citep{ishwaran2008random}.
    \item \textit{Neural discrete-time models:} DeepHit~\citep{lee2018deephit}, DeepSurv~\citep{katzman2018deepsurv}, MTLR~\citep{yu2011learning, fotso2018deep}, Nnet-survival~\citep{biganzoli1998feed, gensheimer2019scalable}, CoxTime~\citep{kvamme2019time}, IWSG~\citep{han2021inverse}, CQRNN~\citep{pearce2022censored}, and BNN-MTLR~\citep{qi2023using}. 
    \item \textit{Neural continuous-time models:} DSM~\citep{nagpal2021dsm}, SuMoNet~\citep{rindt2022survival}, SurvivalMDN~\citep{han2022survival}, DeepAFT-Weibull~\citep{norman2024deepaft}, and DeepAFT-Loglogistic~\citep{norman2024deepaft}. 
\end{itemize}


\textbf{Metrics.}
We evaluate models using five metrics (Appendix~\ref{app:metrics}): IPCW-adjusted integrated Brier score (IBS; probabilistic accuracy)~\citep{graf1999assessment},  concordance index (CI; discrimination)~\citep{harrell1996multivariable}, D-calibration (distributional calibration)~\citep{haider2020effective}, median-time mean absolute error (MAE; time prediction)~\citep{qi2023effective}, and log-rank reliability (agreement with observed time-to-event outcomes).

\textbf{Experimental Protocol.}
For each dataset, we conduct 10 repeated experiments using independent 70\%/30\% train/test splits.
We report the mean $\pm$ standard deviation across the 10 repetitions. 
For aggregate comparison, models are ranked separately within each dataset and metric, with rank 1 assigned to the best-performing method. 
Appendices~\ref{app:compute}-\ref{app:benchmark_protocol} provide further details.

\textbf{RQ1: Predictive Performance.}
Figure~\ref{fig:main_results} summarizes model ranks across all benchmark datasets, with better-performing methods appearing farther to the right. 
SurvivalPFN achieves the strongest overall rank among the compared methods, indicating robust performance across the full benchmark suite. 
Metric-wise, SurvivalPFN is among the leading methods for IBS, MAE, and Log-Rank, showing strong probabilistic survival prediction, accurate time estimation, and strong agreement with observed time-to-event outcomes. 
It also remains competitive on CI and D-calibration, although several other baselines (\eg, RSF, CWGB, DeepSurv) rank higher in these metric-specific rankings. 

Additional stratified results (based on sample size and censoring rate) are provided in Appendix~\ref{app:rq1}. There, SurvivalPFN shows its largest advantage on small datasets (Figure~\ref{fig:best_model_small_data}), while its relative performance decreases as dataset size increases (Figures~\ref{fig:best_model_mid_data} and \ref{fig:best_model_large_data}).
This behavior is consistent with the \textit{motivation of amortized Bayesian survival prediction}: when each downstream dataset contains limited observations, SurvivalPFN can leverage the inductive structure learned during prior-data pretraining rather than fitting a flexible survival model from scratch.
In contrast, across low-, medium-, and high-censoring regimes (Figures~\ref{fig:best_model_low_censor}-\ref{fig:best_model_high_censor}), SurvivalPFN remains consistently strong.
\begin{figure}
    \centering
    \includegraphics[width=\linewidth]{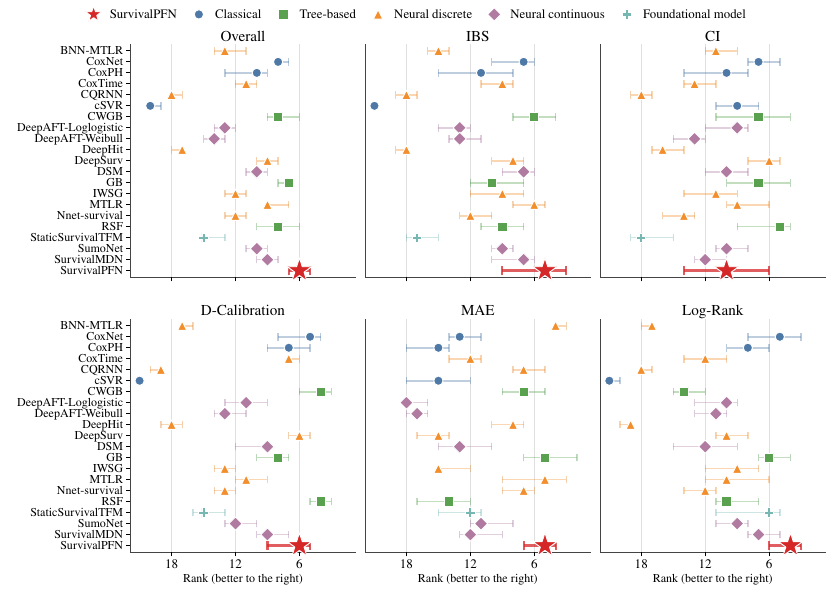}
    \caption{\textbf{Model ranks across 61 benchmark datasets.} Points/stars denote median ranks across datasets, with horizontal bars showing 95\% bootstrap confidence intervals for the median rank. }
    \label{fig:main_results}
\end{figure}

\textbf{RQ2: Computational Efficiency.} 
Figure~\ref{fig:teaser} compares the computational efficiency of SurvivalPFN with all baselines.
We report the total training-plus-inference time across datasets, excluding hyperparameter-tuning time for neural-network-based methods.
SurvivalPFN is highly efficient: it is only modestly slower than CoxNet, the fastest baseline, while achieving the best overall ranking across the five evaluation metrics.
By contrast, tree-based and neural-network-based methods require substantially greater computation.

\begin{wrapfigure}{r}{7.5cm}
   \vspace{-20pt}
    \centering
    \includegraphics[width=\linewidth]{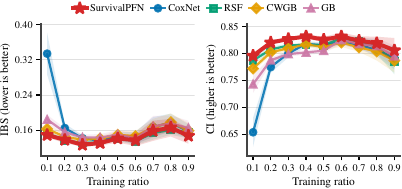}
    \caption{\textbf{Performance on the PBC dataset for SurvivalPFN and top-performing models.} Shaded regions denote standard errors over 10 repeated runs.}
    \label{fig:diff_ratio_main}
    \vspace{-10pt}
\end{wrapfigure}
\textbf{RQ3: Sensitivity to Training-Set Size.}
Figure~\ref{fig:diff_ratio_main} shows how model performance changes as the fraction of training data varies. 
SurvivalPFN shows stable predictive performance across split ratios, maintaining consistently low IBS and high CI even when only a small fraction of the data is used for training. 
In contrast, several baselines are more sensitive to limited training data: CoxNet, GB, and CWGB perform poorly at the smallest ratio and improve markedly as more data become available. 
These trends suggest that SurvivalPFN is comparatively robust in low-data regimes.
Appendix~\ref{app:rq3} and Figure~\ref{fig:diff_ratio_appendix} contain results for more datasets.

\textbf{RQ4: Comparison with General TFMs.}
Figure~\ref{fig:tfm_main} compares SurvivalPFN with general-purpose tabular foundational regressors by training them only on uncensored instances. 
SurvivalPFN achieves the best overall rank and is consistently the top-ranked method across all metrics. This suggests that directly adapting generic TFMs to survival outcomes is insufficient: explicitly pretraining on identifiable right-censored survival tasks yields substantially more reliable survival prediction.

\textbf{RQ5: Ablation Studies.}
Due to space constraints, Appendix~\ref{app:rq5} reports how prior design, query schedules, monotone transformations, and objective functions affect SurvivalPFN performance.

\begin{wrapfigure}{r}{7.5cm}
   \vspace{-40pt}
    \centering
    \includegraphics[width=\linewidth]{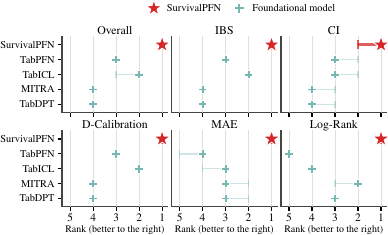}
    \vspace{-10pt}
    \caption{\textbf{Comparison of SurvivalPFN with selected general TFMs across 61 benchmark datasets.} Plotting conventions follow Figure~\ref{fig:main_results}.}
    \label{fig:tfm_main}
    \vspace{-10pt}
\end{wrapfigure}
\section{Related Work}
\label{sec:related_work}
\textbf{Classical and Deep Survival Analysis.}
A broad range of estimators has been developed for right-censored data. 
CoxPH~\citep{cox1972regression} relies on proportional hazards and linear covariate effects, while fully parametric models such as exponential~\citep{mendenhall1958estimation} and Weibull~\citep{peto1973weibull} models impose explicit distributional forms.
Modern machine-learning methods improve flexibility, but introduce other assumptions. 
Neural Cox-based models, such as DeepSurv~\citep{katzman2018deepsurv} and CoxTime~\citep{kvamme2019time}, relax linear covariate effects but retain Cox-style hazard modeling. 
Discrete-time models, including MTLR~\citep{yu2011learning,fotso2018deep} and DeepHit~\citep{lee2018deephit,lee2019dynamic}, depend on a chosen time grid and may become overparameterized with many bins. 
Continuous-time neural models are more flexible, but still impose structure through parametric mixtures~\citep{nagpal2021dsm}, latent-variable models~\citep{ranganath2016deep,miscouridou2018deep}, monotonic density estimators~\citep{rindt2022survival,chilinski2020neural}, or neural ODE hazards~\citep{groha2020general,tang2022soden}.

\textbf{Bayesian Survival Analysis.} Bayesian survival models quantify uncertainty by placing priors over model parameters and integrating over posterior uncertainty. Recent neural variants include BNN-ISD, which uses Bayesian neural networks to obtain credible intervals and supports feature selection~\citep{qi2023using}; Bayesian LSTM-SURV, which combines the survival likelihood with Bayesian mixed-effects updating under a Weibull parametric form~\citep{gao2025neural}; and NeuralSurv, which uses variational inference for Bayesian deep survival prediction~\citep{monod2025neuralsurv}. While these methods demonstrate the value of Bayesian uncertainty quantification, they remain tied to method-specific assumptions and per-dataset inference or updating.

\textbf{Concurrent Work on TFMs for Survival Analysis.}
Two concurrent works also explore TFMs for survival analysis. 
\citet{kim2026tabular} convert survival prediction to a sequence of binary classification tasks over discretized time points, enabling off-the-shelf TFMs without survival-specific pretraining. 
This simple reduction avoids new pretraining, but expands each instance into many time-indexed examples; as a result, context length scales with the number of bins, larger datasets require subsampling, and performance can suffer when the TFM sees only a compressed view of the risk set. 
\citet{seletkov2026survival} instead pretrain a survival-specific in-context model on synthetic data from parametric extended-hazard mechanisms. 
This is closer to SurvivalPFN, but its prior is limited to parametric hazard families and random censoring ($E \perp C$), potentially limiting broader use.
In contrast, SurvivalPFN uses a broader family of identifiable right-censored DGPs, supports covariate-dependent censoring under conditional independence, avoids explicit parametric structure, and provides a posterior-predictive consistency argument. 
We include \citet{kim2026tabular}'s static formulation as \textsc{StaticSurvivalTFM}; SIC is not directly compared because public weights/code are not yet available.

\section{Conclusions, Limitations, and Future Work}
\label{sec:conclusion}

We introduced SurvivalPFN, a prior-data fitted network for amortized Bayesian survival prediction from right-censored data. 
By pretraining on diverse, identifiable survival data-generating processes, SurvivalPFN produces posterior predictive survival distributions in a single forward pass, without dataset-specific training or hyperparameter tuning. 
Across 61 real-world datasets, SurvivalPFN achieves strong overall performance while remaining computationally efficient, suggesting that PFN-style in-context learning is a promising foundation for flexible survival prediction. 
Because survival models already inform high-stakes decisions -- from clinical risk scores~\citep{lip2010refining} to resource-allocation policies~\citep{kim2021meld} -- we believe that advances in accuracy, efficiency, and uncertainty quantification for survival models can generate substantial human benefit.

Several limitations remain. 
First, SurvivalPFN relies on conditional independent censoring to preserve identifiability; under dependent censoring, the event-time distribution is not nonparametrically identifiable from observed data alone, and our method is not valid. 
Extending SurvivalPFN to identifiable dependent-censoring priors, such as copula methods~\citep{gharari2023copula,zhang2024deep}, is an important direction. 
Additionally, SurvivalPFN inherits the size-scalability trade-off of PFN-style models, with reduced relative performance on larger tables~\citep{hollmann2025accurate,balazadeh2025causalpfn}; improving long-context inference is left for future work.

\bibliographystyle{plainnat}
\bibliography{neurips_2026}


\newpage
\appendix

\addtocontents{toc}{\protect\setcounter{tocdepth}{2}}

\renewcommand{\contentsname}{Appendix Contents}
\tableofcontents

\newpage

\section{Notation}
\label{app:notation}

Table~\ref{tab:notation-summary} summarizes all the notation used throughout the paper.
Following the convention, we use uppercase letters for random variables and lowercase letters for realizations.

\begin{longtable}{p{0.27\linewidth}p{0.67\linewidth}}
\caption{Summary of notation in the paper.}
\label{tab:notation-summary}
\\
\toprule
\textbf{Notation} & \textbf{Definition} \\
\midrule
\endfirsthead

\toprule
\textbf{Notation} & \textbf{Definition} \\
\midrule
\endhead

\midrule
\multicolumn{2}{r}{\emph{Continued on next page}} \\
\endfoot

\bottomrule
\endlastfoot

$\mathcal{X}$, $\mathbb{R}_{+}$
& Covariate space and nonnegative time domain. \\

$i \in \{1,\ldots,N\}$
& Index for an individual observation. \\

$d$
& Number of covariates/features. \\

$X \in \mathcal{X}$
& Covariates. \\

$E \in \mathbb{R}_{+}$
& Latent event time of interest. \\

$C \in \mathbb{R}_{+}$
& Latent censoring time. \\

$T = \min(E,C)$
& Observed follow-up time. \\

$\Delta = \mathbbm{1}\{E \le C\}$
& Event indicator: $\Delta=1$ for observed event and $\Delta=0$ for right-censored. \\

$(x_i,t_i,\delta_i)$
& Observed right-censored tuple. \\

$(e_i,c_i)$
& Latent event and censoring times. \\

$\mathscr{D}$
& Random observed dataset. \\

$\data=\{(x_i,t_i,\delta_i)\}_{i=1}^{N}$
& Realization of $\mathscr{D}$. \\

$\data^{tr}_{\theta}$, $\data^{te}_{\theta}$
& Synthetic context/training set and query/test set sampled from DGP parameter $\theta$ during prior-data pretraining. \\

$\omega$ 
& Trainable parameters of SurvivalPFN. \\

$\theta \in \Theta$
& Latent parameter indexing a survival data-generating process. \\

$\theta^\ast$
& True DGP parameter for a downstream dataset. \\

$\pi(\cdot)$
& Prior distribution over survival DGPs used to generate synthetic training tasks. \\

$P^{\theta}$
& Joint law of $(X,E,C,T,\Delta)$ under DGP parameter $\theta$. \\

$P^{\theta}_{\mathrm{obs}}$
& Observational law of $(X,T,\Delta)$ induced by $P^{\theta}$. \\

$f^{\theta}_{E\mid X}(t\mid x)$
& Conditional density function for event time, under $\theta$. \\

$F^{\theta}_{E\mid X}(t\mid x)$
& Conditional CDF for event time, $P^{\theta}(E\leq t\mid X=x)$. \\

$S^{\theta}_{E\mid X}(t\mid x)$
& Conditional survival function for event time, $P^{\theta}(E>t\mid X=x)=1-F^{\theta}_{E\mid X}(t\mid x)$. \\

$\lambda^{\theta}_{E \mid X}(t\mid x)$
& Conditional hazard function,
$\lambda^{\theta}_{E}(t\mid x)=
f^{\theta}_{E\mid X}(t\mid x)/
S^{\theta}_{E\mid X}(t\mid x)$. \\

$S^{\theta}_{C\mid X}(t\mid x)$
& Conditional survival function for censoring time, $P^{\theta}(C\ge t\mid X=x)$. \\

$\lambda^{\theta}_{C\mid X}(t\mid x)$
& Conditional censoring hazard. \\

$f_{E\mid X,\mathscr{D}}(t\mid x,\data)$
& Posterior predictive event-time distribution. \\

$S_{E\mid X,\mathscr{D}}(t\mid x,\data)$
& Posterior predictive survival distribution (PPSD). \\

$x^\ast$
& Query covariates for a new individual. \\

$\widetilde{\delta}^\ast$
& Query indicator supplied to SurvivalPFN. \\

$q_{\omega}(\cdot\mid x^\ast,\widetilde{\delta}^\ast,\data)$
& SurvivalPFN's predictive distribution over discretized time bins. \\

$L$
& Number of time bins used to represent the predictive distribution. \\

$\{\mathcal I_\ell\}_{\ell=1}^{L}$
& SurvivalPFN's transformed-time bins used to represent the discretized predictive distribution. \\

$\mathcal{G}=\{t_1<\cdots<t_m\}$
& Discrete time grid for benchmarking discrete-time survival models. \\

$\mathrm{CV}(T)$
& Coefficient of variation of observed times:
$\mathrm{CV}(T)=\operatorname{sd}(T)/\mathbb{E}[T]$, when used in DGP diagnostics. 

\end{longtable}

\section{Identifiability and Non-identifiability under Right Censoring}
\label{app:identifiable}

In the context of survival analysis and causal inference, identifiability is the prerequisite for learning. If a model is non-identifiable, infinite data cannot distinguish between multiple underlying ``truths'' (\eg, whether a drug works or if patients are simply dropping out due to side effects). Without identifiability, no consistent estimator exists, and any conclusions drawn from the data rely entirely on untestable assumptions rather than empirical evidence.

\citet[Theorem 2]{tsiatis1975nonidentifiability} first proved that the latent joint distribution of event time $E$ and censoring time $C$ is not identifiable from the (infinite) observed data $(T, \Delta)$ without additional assumptions.
Formally, we restate the theorem using the notation that is consistent within this paper, by considering $E$ and $C$ as two competing events:
\begin{theorem}[Non-identifiability; Marginal]
Let $S_{E, C}(e, c) = P(\, E > e, C > c \,)$ be an arbitrary joint survival function where $E$ and $C$ are dependent. There exists a different joint survival function $S_{E, C}^\ast(e, c)$, constructed such that $E$ and $C$ are independent -- $S_{E, C}^\ast(e, c) = S_E^\ast(e) \, S_C^\ast(c)$ -- which generates the exact same observed data distribution $P(T, \Delta)$.
\end{theorem}

This means, without the assumption of random censoring ($E \perp C$), the marginal survival function $S_{E\mid X}(t)$ cannot be uniquely determined.
This theorem can be easily extend to the conditional setting:
\begin{theorem}[Non-identifiability; Conditional]
\label{thm:nonident_conditional}
Let $X$ be a set of covariates. Let the true conditional joint survival function be $S_{E, C \mid X}(e, c \mid x) = P(E > e, C > c \mid X=x)$, where $E \not\perp C \mid X$.
For any such dependent model, there exists a valid conditional independent model $S_{E, C \mid X}^\ast(e, c \mid x) = S_{E\mid X}^\ast(e \mid x) S_{C\mid X}^\ast(c \mid x)$ such that the observable distributions of $(T, \Delta \mid X)$ are identical.
\end{theorem}
\begin{proof}
While the proof for this conditional non-identifiability is straightforward via following~\citet{tsiatis1975nonidentifiability}'s step, we present the proof for completeness.

Let the observed data be characterized by the conditional sub-survival functions:
\begin{align*}
    \Phi(t \mid x) \ &= \ P(T > t, \Delta = 1 \mid x) \\
    \Psi(t \mid x) \ &= \ P(T > t, \Delta = 0 \mid x)
\end{align*}
These functions completely describe the likelihood of the observed data.

We construct a ``proxy'' independent world, denoted by a superscript $^\mathrm{ind}$, by defining its conditional hazards to match the observed cause-specific hazards of the original world:
\begin{align*}
    \lambda_{E\mid X}^{\mathrm{ind}}(t \mid x) \ &= \ \lim_{d t \rightarrow 0} \frac{\Pr^\mathrm{ind}(t\leq T < t + d t , \Delta=1\mid T \geq t, X=x)}{d t} \ = \ \frac{-\frac{\partial}{\partial t} \Phi(t \mid x)}{\Phi(t \mid x) + \Psi(t \mid x)} \\
    \lambda_{C\mid X}^\mathrm{ind}(t \mid x) \ &= \ \lim_{d t \rightarrow 0} \frac{\Pr^\mathrm{ind}(t\leq T < t + d t , \Delta=0\mid T \geq t, X=x)}{d t} \ = \ \frac{-\frac{\partial}{\partial t} \Psi(t \mid x)}{\Phi(t \mid x) + \Psi(t \mid x)}
\end{align*}
We define the marginal survival functions in the proxy world as:
\begin{align*}
    S_{E\mid X}^\mathrm{ind}(t \mid x) \ &= \ \exp\left( -\int_0^t \lambda_{E\mid X}^\mathrm{ind}(u \mid x) du \right), \\
    S_{C\mid X}^\mathrm{ind}(t \mid x) \ &= \ \exp\left( -\int_0^t \lambda_{C\mid X}^\mathrm{ind}(u \mid x) du \right).
\end{align*}

And the joint distribution (with conditional independence):
\begin{align*}
    S_{E, C \mid X}^\mathrm{ind}(e, c \mid x) \ = \ S_{E\mid X}^\mathrm{ind}(e \mid x) \, S_{C\mid X}^\mathrm{ind}(c \mid x)
\end{align*}
To verify that this proxy world generates the same data $\Phi(t \mid x)$, we calculate the probability of observing an event in the proxy world:
\begin{align*}
    \Phi^\mathrm{ind}(t \mid x) &= \int_t^\infty f_{E\mid X}^\mathrm{ind}(u \mid x) S_{C\mid X}^\mathrm{ind}(u \mid x) du \\
    &= \int_t^\infty \lambda_{E\mid X}^\mathrm{ind}(u \mid x) S_{E\mid X}^\mathrm{ind}(u \mid x) S_{C\mid X}^\mathrm{ind}(u \mid x) du \\
    &= \int_t^\infty \lambda_{E\mid X}^\mathrm{ind}(u \mid x) S^\mathrm{ind}(u, u \mid x) du
\end{align*}
Substituting the definitions of $\lambda_{E\mid X}^\mathrm{ind}$ and noting that $S^\mathrm{ind}(u, u \mid x) = \Phi(u \mid x) + \Psi(u \mid x)$ (the overall survival probability matches the sum of sub-survival functions):
\begin{align*}
    \Phi^\mathrm{ind}(t \mid x) &= \int_t^\infty \left( \frac{-\frac{\partial}{\partial u} \Phi(u \mid x)}{\Phi(u \mid x) + \Psi(u \mid x)} \right) (\Phi(u \mid x) + \Psi(u \mid x)) du \\
    &= \int_t^\infty -\frac{\partial}{\partial u} \Phi(u \mid x) du \\
    &= \Phi(t \mid x)
\end{align*}
Since $\Phi^\mathrm{ind}(t \mid x) = \Phi(t \mid x)$ (and by symmetry $\Psi^\mathrm{ind}(t \mid x) = \Psi(t \mid x)$), the independent model $S^\mathrm{ind}$ is indistinguishable from the true dependent model $S$.
\end{proof}

While Theorem~\ref{thm:nonident_conditional} states that we cannot distinguish \textit{dependent} from \textit{independent} censoring using data alone, the corollary below establishes that if we are willing to assume independence (either marginal or conditional), the latent event distribution becomes identifiable.

\begin{corollary}[Identifiability under Independence]
Suppose we restrict our attention to the class of models that satisfy conditional independent censoring, $E \perp C \mid X$ (including $E\perp C$).
Then, the marginal survival function of the event, $S_{E\mid X}(t \mid x)$, is uniquely identifiable from the observed data distribution.
That is, if two models $S$ and $S^\mathrm{ind}$ both satisfy conditional independence but have different event marginals ($S_{E\mid X} \neq S_{E\mid X}^\mathrm{ind}$), they must generate distinct observed data distributions.
\end{corollary}

\begin{proof}
We prove this by contradiction.
Let $\lambda_{E\mid X}(t \mid x)$ denote the \textit{cause-specific} hazard derived purely from the data $(X, T, \Delta)$:
\begin{align*}
    \lambda_{E\mid X}(t \mid x) \ &= \ \lim_{d t \rightarrow 0} \frac{P(t\leq T < t + d t , \Delta=1 \mid T \geq t, X=x)}{d t}
\end{align*}
Let $\lambda(t \mid x)$ denote the \textit{net} hazard of the event of interest:
\begin{align*}
    \lambda(t \mid x)  \ &= \ \lim_{d t \rightarrow 0} \frac{P(t\leq E < t + d t \mid E \geq t, X=x)}{d t}
\end{align*}
Under the assumption of conditional independence ($E \perp C \mid X$), standard survival theory dictates that the net hazard is equal to the cause-specific hazard: $\lambda_{E\mid X}(t \mid x) = \lambda(t \mid x)$.
Since the hazard function uniquely defines the survival function via $S_{E\mid X}(t \mid x) = \exp(-\int_0^t \lambda_{E\mid X}(u \mid x) du)$, the latent distribution $S_{E\mid X}$ is uniquely determined by the observed function $\lambda_{E\mid X}$.

Now, consider two models with different marginals, $S_{E\mid X}^{(1)}(t \mid x) \neq S_{E\mid X}^{(2)}(t \mid x)$.
This inequality implies their net hazards must differ: $\lambda^{(1)}(t \mid x) \neq \lambda^{(2)}(t \mid x)$.
By the equality derived above, their observed cause-specific hazards must also differ: $\lambda_{E\mid X}^{(1)}(t \mid x) \neq \lambda_{E\mid X}^{(2)}(t \mid x)$.
Different hazards imply different observed data distributions. Thus, the model is identifiable.
\end{proof}

These properties imply that while SurvivalPFN can be robustly trained under assumptions of marginal or conditional independence, it cannot learn dependent censoring mechanisms from observed data alone.

\section{Posterior-Predictive Consistency}
\label{app:survivalpfn_consistency}

This appendix formalizes Proposition~\ref{prop:informal-survival-consistency}.
The result concerns the Bayesian posterior predictive survival distribution
(PPSD) defined in \Eqref{eq:PPSD}. It shows that, under an identifiable
survival prior, the Bayesian PPSD is asymptotically consistent for the true
conditional event-time survival function. We then state the corresponding
idealized implication for SurvivalPFN when the transformer exactly amortizes
this Bayesian target.

\subsection{Notation and regularity assumptions}

Recall that a survival data-generating process (DGP) $P^\theta(X,E,C,T,\Delta)$ is indexed by $\theta\in\Theta$.
We write $P^\theta_{\mathrm{obs}}$ for the marginal distribution over the observable random variables $(X,T,\Delta)$. 
We still use $\pi(\cdot)$ to denote the prior over $\Theta$ induced by the
synthetic prior-data generator used to pretrain SurvivalPFN.

For each $\theta$, let
\[
    S_{E\mid X,\Theta}(t\mid x,\theta) \ := \ \Pr(E>t\mid X=x, \Theta = \theta)
\]
denote the conditional event-time survival function. For an observed dataset $\data=\{(X_i,T_i,\Delta_i)\}_{i=1}^N$, where $(X_i,T_i,\Delta_i)  \overset{\mathrm{i.i.d.}}{\sim} P^\theta_{\mathrm{obs}}$, and a query covariate vector $x^\ast$,
the Bayesian PPSD is (repeated from \Eqref{eq:PPSD})
\[
    S_{E\mid X,\mathscr{D}}(t\mid x^\ast,\data) \ = \ 
    \int_\Theta
    S_{E\mid X,\Theta}(t\mid x^\ast,\vartheta)
    f_{\Theta\mid\mathscr{D}}(\vartheta \mid \data)
    \,d\vartheta .
\]

\begin{assumption}[Regularity]
\label{assump:regularity}
We assume the following standard regularity conditions.
\begin{enumerate}
    \item $(\Theta,\mathcal{B}_\Theta)$ is a standard Borel parameter space.
    \item The maps $\theta\mapsto P^\theta$ and
    $\theta\mapsto P^\theta_{\mathrm{obs}}$ are measurable.
    \item The image set $\{P^\theta_{\mathrm{obs}}:\theta\in\Theta\}$ is a Borel subset
    of the space of probability measures over $(X,T,\Delta)$.
    \item For each time $t$ in the evaluation region, there exists a version of
    $S_{E\mid X,\Theta}(t\mid x^\ast,\theta)$ that is jointly measurable in
    $(x^\ast,\theta)$.
\end{enumerate}
\end{assumption}

These assumptions are technical rather than substantive. They ensure that
priors, posteriors, conditional expectations, and the quotient construction
below are well-defined. They are satisfied by the usual finite-dimensional
parameter spaces and measurable simulators used in statistical and machine
learning models.

\subsection{Observed-law equivalence and survival identifiability}

The full latent law $P^\theta(X,E,C,T,\Delta)$ is generally not identifiable
from right-censored observations. The data can identify only the observed law
$P^\theta_{\mathrm{obs}}$ over $(X,T,\Delta)$. We therefore group DGP parameters by
observational equivalence:
\[
    \theta_1 \ \sim \ \theta_2
    \qquad\Longleftrightarrow\qquad
    P^{\theta_1}_{\mathrm{obs}} \ = \ P^{\theta_2}_{\mathrm{obs}}.
\]
Let
\[
    \mathcal{Q} \ := \ \Theta/\!\sim
\]
denotes the set of observational quotient space (equivalence classes) induced by $\sim$,
and let $[\theta]\in\mathcal{Q}$ denote the equivalence class of $\theta$. 
Each class $[\theta]$ contains all latent survival DGPs that are indistinguishable.

\begin{definition}[Survival-identifiable prior]
\label{def:survival_identifiable_prior}
Fix a time $t$ in the evaluation region. We say that the prior $\pi$ is
\emph{survival-identifiable at time $t$} if there exists a measurable map
\[
    F_t: \mathcal{X}\times \mathcal{Q} \ \to \ [0,1]
\]
such that, for $\pi$-almost every $\theta$ and $P_X^\pi$-almost every $x$ (except where $\theta$ or $x$ has probability $0$),\footnote{In the following, we will just use the phrase ``every $\theta$'' and ``every $x$'' for simplicity.}
\[
    S_{E\mid X,\Theta}(t\mid x,\theta) \ = \ F_t(x,[\theta]).
\]
Equivalently, within the support of the prior, any two DGPs that induce the
same observational distribution over $(X,T,\Delta)$ must also induce the same conditional event-time survival function at time $t$.
\end{definition}

The conditional independent censoring and positivity assumptions discussed in
Section~\ref{sec:background} provide sufficient conditions for this
definition: under those assumptions, $S_{E\mid X}(t\mid x)$ is a functional of
the observed law $P(X,T,\Delta)$ on the identifiable time region.

\subsection{Consistency of the Bayesian PPSD}

\begin{proposition}[Formal consistency]
\label{prop:formal_consistency}
Fix a time $t$ in the evaluation region. Under
Assumption~\ref{assump:regularity}, there exist sets
$\mathcal{X}_0\subseteq\mathcal{X}$ and
$\Theta_0\subseteq\Theta$ with
\[
    P_X^\pi(\mathcal{X}_0) \ = \ 1,
    \qquad
    \pi(\Theta_0) \ = \ 1,
\]
such that for every $x^\ast \in\mathcal{X}_0$ and every
$\theta^\ast\in\Theta_0$, 
then
\begin{align}
\label{eq:consistency}
    S_{E\mid X,\mathscr{D}}(t\mid x^\ast,\data)
    \ \xrightarrow[N\to\infty]{\mathrm{a.s.}} \ 
    S_{E\mid X,\Theta}(t\mid x^*,\theta^\ast)
\end{align}
if and only if the prior $\pi$ is survival-identifiable at time $t$.

For any finite or countable evaluation grid
$\mathcal{G}$, the same result holds simultaneously for all
$t\in\mathcal{G}$ by intersecting the corresponding full-measure sets.
\end{proposition}

\begin{proof}
We first define the quotient-level target.
For fixed $t$ and $x$, we define
\[
    M_t(x^\ast,[\theta]) \ := \
    \mathbb{E}_\pi
    \left[
        S_{E\mid X,\Theta}(t\mid x^\ast,\vartheta)
        \mid
        [\vartheta]=[\theta]
    \right].
\]
This is the prior-average survival probability among all DGPs that induce the
same observed law as $\theta$. Since
$0\le S_{E\mid X,\Theta}(t\mid x^\ast,\vartheta)\le 1$, this conditional expectation
is integrable.

By construction, two different elements of $\mathcal{Q}$ correspond to two
different observational distributions. 
Thus, the quotient parameter $[\theta]$ is identifiable from the observed distribution. 
Assumption~\ref{assump:regularity} ensures that the quotient model is measurable and that Doob's consistency theorem~\citep{doob1949application} applies to posterior expectations of integrable functions on this quotient space.

Therefore, for every true parameter $\theta^\ast$ (except where $\theta$ has probability $0$), if
$\data\sim P^{\theta^\ast}_{\mathrm{obs}}$, then
\begin{align}
\label{eq:doob_consistency}
    \mathbb{E}_\pi
    \left[
        M_t(x^\ast,[\theta])
        \mid
        \data
    \right]
    \ \xrightarrow[N\to\infty]{\mathrm{a.s.}} \
    M_t(x^\ast,[\theta^\ast]).
\end{align}

We now show that the left-hand side of
\Eqref{eq:doob_consistency} is the Bayesian PPSD. Since the observed
data distribution depends on $\theta$ only through the equivalence class
$[\theta]$, we have the conditional independence
\[
    \data \perp \theta \mid [\theta].
\]
Hence,
\[
\begin{aligned}
    \mathbb{E}_\pi \left[M_t(x^\ast,[\theta]) \mid \data \right] \
    &= \ \mathbb{E}_\pi
    \left[ \mathbb{E}_\pi
        \left[ S_{E\mid X,\Theta}(t\mid x^\ast,\theta) \mid [\theta] \right]
        \mid \data
    \right] \\
    &= \ \mathbb{E}_\pi
    \left[
        \mathbb{E}_\pi
        \left[ S_{E\mid X,\Theta}(t\mid x^\ast,\theta) \mid [\theta],\data \right]
        \mid \data
    \right] \\
    &= \ \mathbb{E}_\pi
    \left[
        S_{E\mid X,\Theta}(t\mid x^\ast,\theta) \mid \data
    \right] \\
    &= \ \int_\Theta
    S_{E\mid X,\Theta}(t\mid x^\ast,\vartheta)
    f_{\Theta\mid\mathscr{D}}(\vartheta\mid\data)
    \,d\vartheta \\
    &= \ S_{E\mid X,\mathscr{D}}(t\mid x^\ast,\data).
\end{aligned}
\]
Combining this equality with \Eqref{eq:doob_consistency} gives
\begin{align}
\label{eq:ppsd_limit_class_average}
    S_{E\mid X,\mathscr{D}}(t\mid x^\ast,\data)
    \ \xrightarrow[N\to\infty]{\mathrm{a.s.}} \
    M_t(x^\ast,[\theta^\ast]).
\end{align}
Without identifiability, this is the strongest possible conclusion: the PPSD
converges to the prior-average survival function over DGPs that are
observationally equivalent to the truth.

Finally, note $\pi$ is survival-identifiable at time $t$. 
Then, by Definition~\ref{def:survival_identifiable_prior}, all DGPs in the same
observational equivalence class have the same survival probability at
$(t,x^\ast)$. Therefore,
\begin{align*}
    M_t(x^\ast,[\theta^\ast]) \
    &= \ \mathbb{E}_\pi
    \left[
        S_{E\mid X,\Theta}(t\mid x^\ast,\vartheta)
        \mid
        [\vartheta]=[\theta^\ast]
    \right] \\
    &= \ S_{E\mid X,\Theta}(t\mid x^\ast,\theta^\ast).
\end{align*}
Substituting this into \Eqref{eq:ppsd_limit_class_average} proves
\Eqref{eq:consistency}. This establishes sufficiency.

For necessity, suppose that the Bayesian PPSD is consistent in the sense of
\Eqref{eq:consistency} for every $\theta^\ast$. From
\Eqref{eq:ppsd_limit_class_average}, the same sequence also converges almost
surely to $M_t(x^\ast,[\theta^\ast])$. By uniqueness of almost-sure limits,
\[
    S_{E\mid X,\Theta}(t\mid x^\ast,\theta^\ast) \ = \ M_t(x^\ast,[\theta^\ast])
\]
for every $\theta^\ast$. Since the right-hand side depends on
$\theta^\ast$ only through its observational equivalence class
$[\theta^\ast]$, the survival functional
$S_{E\mid X,\Theta}(t\mid x^\ast,\theta^\ast)$ also depends only on
$[\theta^\ast]$. Thus, Definition~\ref{def:survival_identifiable_prior} holds
with $F_t(x^\ast,[\theta])=M_t(x^\ast,[\theta])$, up to arbitrary definition on null sets. 
Hence the prior is survival-identifiable at time $t$.
\end{proof}

\subsection{Implication for SurvivalPFN}

Proposition~\ref{prop:formal_consistency} characterizes the
Bayesian posterior predictive target. SurvivalPFN is a finite neural network
trained to approximate this target, so the theorem does not by itself prove
finite-sample or finite-capacity consistency of the trained transformer. It does,
however, imply consistency in the idealized exact-amortization limit.

\begin{corollary}[Consistency of SurvivalPFN]
\label{cor:survivalpfn-consistency}
Assume the conditions of
Proposition~\ref{prop:formal_consistency}. Suppose that, for
$\widetilde{\delta}^\ast=1$, SurvivalPFN exactly amortizes the Bayesian posterior
predictive event-time distribution, so that its predicted survival curve
satisfies
\[
    \widehat{S}_{\omega}(t\mid x^\ast,\data) \ =\ S_{E\mid X,\mathscr{D}}(t\mid x^\ast,\data).
\]
Then, for every true DGP parameter $\theta^\ast$ and every $x$,
\[
    \widehat{S}_{\omega}(t\mid x^\ast,\data)
    \ \xrightarrow[N\to\infty]{\mathrm{a.s.}} \
    S_{E\mid X,\Theta}(t\mid x^\ast,\theta^\ast).
\]
More generally, the same conclusion holds if the amortization error vanishes:
\[
    \left|
        \widehat{S}_{\omega}(t\mid x^\ast,\data) - S_{E\mid X,\mathscr{D}}(t\mid x^\ast,\data)
    \right|
    \ \xrightarrow[N\to\infty]{} \ 0 .
\]
\end{corollary}

\begin{proof}
The exact-amortization case follows immediately by substituting
$\widehat{S}_{\omega}(t\mid x^\ast,\data) = S_{E\mid X,\mathscr{D}}(t\mid x^\ast,\data)$ into
Proposition~\ref{prop:formal_consistency}. The vanishing-error case
follows from the triangle inequality:
\begin{align*}
    & \left|
        \widehat{S}_{\omega}(t\mid x^\ast,\data) - S_{E\mid X,\Theta}(t\mid x^\ast,\theta^\ast)
    \right|
    \\
    &\qquad\leq \
    \left|
        \widehat{S}_{\omega}(t\mid x^\ast,\data)
        -
        S_{E\mid X,\mathscr{D}}(t\mid x^\ast,\data)
    \right|
    +
    \left|
        S_{E\mid X,\mathscr{D}}(t\mid x^\ast,\data)
        -
        S_{E\mid X,\Theta}(t\mid x^\ast,\theta^\ast)
    \right|.
\end{align*}
The first term vanishes by assumption, and the second term vanishes almost
surely by Proposition~\ref{prop:formal_consistency}.
\end{proof}

\section{Additional SurvivalPFN Details}
\label{app:survivalpfn_details}

\subsection{Synthetic Prior and DGP Simulation}
\label{app:prior_generation}

SurvivalPFN is pretrained on synthetic right-censored survival datasets sampled from a prior $\pi(\cdot)$ over identifiable survival data-generating processes. 
A draw $\theta\sim\pi(\cdot)$ specifies all random choices needed to generate one synthetic task, including the covariate distribution, event-time mechanism, censoring-time mechanism, censoring type, and target censoring rate. 

The simulator returns both the observed survival dataset
\[
    \data^{tr}_{\theta} \ = \ \{(x_i,t_i,\delta_i)_\theta\}_{i=1}^{N},
\]
and the corresponding latent event and censoring times $\{(e_i,c_i)_\theta\}_{i=1}^{N}$, which are used only for constructing supervised query targets during prior-data training.

\paragraph{Random Tabular Generators.}
All survival priors are built from a generic tabular generator. We write $G$ for a random table generator that can sample either an unconditional table or a conditional table. Concretely, a call to $G$ first samples generator-specific latent parameters $\zeta$, and then produces either
\[
    X \ \sim \ G_\zeta(\cdot),
    \qquad
    Y \mid X \ \sim \ G_\zeta(\cdot\mid X).
\]
based on the requirement. 

This abstraction also lets the same survival-prior design use different tabular generators, including PFN-style random multilayer perceptrons, random structural causal models, tree-based generators, or mixtures of these sources.

\paragraph{Censoring Mechanisms.}
We include several censoring mechanisms to cover common right-censoring regimes while retaining identifiability.

\textit{Uniform censoring.}
Event times are generated from the event-time prior, while censoring times are sampled from a uniform distribution. Depending on the time-generation family, the support is either based on the generated event-time range or on a sampled global time horizon:
\[
    C_i \ \sim \ \operatorname{Unif}\!\left(\min_j E_j,\,\max_j E_j\right),
    \qquad \text{or} \qquad
    C_i \ \sim \  \operatorname{Unif}(0,t_{\max}).
\]

\textit{Random censoring.}
Event times are generated conditionally on $X$, while censoring times are generated from a separate unconditional tabular mechanism:
\[
    E\mid X \ \sim \ \Pr(E\mid X),
    \qquad
    C \ \sim \ \Pr(C).
\]
This produces censoring distributions that are more diverse than simple uniform censoring while remaining independent of the event-time noise conditional on the task.

\textit{Administrative censoring.}
Administrative censoring simulates staggered study entry with a common study end date. The simulator samples entry times $A_i$, chooses a fixed administrative end time $a_\ast$, and sets
\[
    C_i \ = \ a_\ast - A_i.
\]
This captures a common survival-analysis setting in which subjects enter at different calendar times but are all censored at the same study termination date.

\textit{Independent censoring.}
For independent (or covariate-dependent) censoring, both event and censoring times are generated conditional on $X$, but from independent blocks:
\[
    E\perp C\mid X,
    \qquad
    \Pr(E,C\mid X) \ = \ \Pr(E\mid X)\,\Pr(C\mid X).
\]
This mechanism is central to the identifiable-prior construction: it allows censoring to depend on covariates while preserving the standard conditional independent censoring assumption used for nonparametric identification.

\paragraph{Prior Family 1: Naive Survival Prior.}
The simplest prior treats generated table values as raw event and censoring times. Let $G_{\zeta}$ denote a conditional event/censoring table generator. For the event-time branch, we sample
\[
    E \ = \ G_{\zeta}(X;1),
\]
where $G_{\zeta}(X;1)$ denotes one conditional output column given $X$. The censoring branch is then chosen according to the censoring type: uniform censoring samples $C$ from a uniform distribution; tabular censoring samples $C$ from an unconditional table generator; administrative censoring constructs $C$ from entry times; and the conditional independent branch samples separate event and censoring columns conditional on $X$. This prior is intentionally simple and flexible: it exposes the model to irregular, nonparametric time patterns without imposing an explicit survival-time family.

\paragraph{Prior Family 2: Survival-Distribution Prior.}
The second prior generates smooth distributions using random monotone maps. For each dataset, the simulator samples a time horizon $t_{\max}>0$, a number of knots $K$, and unconstrained coefficients
\[
    c \ = \ (c_1,\ldots,c_K) \ \in \ \mathbb{R}^{K}.
\]
These coefficients are converted into positive increments
\[
    \Delta_j \ = \
    \frac{\exp(c_j)}
    {\sum_{\ell=1}^{K}\exp(c_\ell)},
    \qquad j=1,\ldots,K,
\]
which define monotone control points
\[
    b_0 \ = \ 0,
    \qquad
    b_j \ =\ \sum_{\ell=1}^{j}\Delta_\ell,
    \qquad j=1,\ldots,K.
\]
The resulting monotone Bernstein map is
\[
    f_c(u) \ = \
    \sum_{j=0}^{K}
    b_j {K\choose j} u^j(1-u)^{K-j},
    \qquad u\in[0,1].
\]
A raw time is then sampled by pushing uniform noise through this monotone map:
\[
    \tau \ = \ t_{\max} f_c(U),
    \qquad
    U\ \sim\ \operatorname{Unif}(0,1).
\]
This construction can be viewed as a random smooth quantile-like function. It gives a flexible family of event-time and censoring-time distributions without committing to a fixed parametric hazard form. Under conditional independent censoring, event and censoring coefficient blocks are generated independently given $X$:
\begin{align*}
    E_i\  &= \ t_{\max} f_{c_i^E}(U_i^E), \\
    C_i \ &= \  t_{\max} f_{c_i^C}(U_i^C),
\end{align*}
with independent $U_i^E,U_i^C\sim\operatorname{Unif}(0,1)$ and independent coefficient blocks conditional on $x_i$.

\paragraph{Prior Family 3: Mixture Prior.}
The third prior samples event and censoring times from mixtures of distributions. In our implementation, the component family is sampled from Weibull and lognormal distributions, and the number of components $K_{\mathrm{mix}}$ is sampled from a finite candidate set. For each row $i$, a table generator emits hidden values
\[
    h_i \ = \    (a_{i1},b_{i1},r_{i1},\ldots,a_{iK_{\mathrm{mix}}},b_{iK_{\mathrm{mix}}},r_{iK_{\mathrm{mix}}}).
\]
The logits $r_{ij}$ define mixture weights
\[
    \pi_{ij} \ = \
    \frac{\exp(r_{ij})}
    {\sum_{\ell=1}^{K_{\mathrm{mix}}}\exp(r_{i\ell})},
    \qquad
    \sum_{j=1}^{K_{\mathrm{mix}}}\pi_{ij} \ = \ 1,
\]
and a component index is sampled as
\[
    Z_i \ \sim \ \operatorname{Categorical}(\pi_{i1},\ldots,\pi_{iK_{\mathrm{mix}}}).
\]
For a Weibull component, positive shape and scale parameters are obtained by
\begin{align*}
    \kappa_{ij} \ &=\ \operatorname{softplus}(a_{ij})+0.1,\\
    \lambda_{ij}\ &=\ \operatorname{softplus}(b_{ij})+0.1,
\end{align*}
and the sampled time is
\[
    \tau_i \ = \
    \lambda_{iZ_i}
    \left[-\log(1-U_i)\right]^{1/\kappa_{iZ_i}},
    \qquad
    U_i \ \sim \ \operatorname{Unif}(0,1).
\]
For a lognormal component,
\begin{align*}
    \mu_{ij} \ &= \ \operatorname{softplus}(a_{ij}), \\
    \sigma_{ij} \ &= \ \operatorname{softplus}(b_{ij}),
\end{align*}
and
\[
    \log \tau_i
    =
    \mu_{iZ_i}
    +
    \sigma_{iZ_i}\varepsilon_i,
    \qquad
    \varepsilon_i\sim\mathcal N(0,1).
\]
This mixture prior provides a complementary source of smooth positive-time distributions with heavy tails and multimodality.

\paragraph{Prior Family 4: Kitchen-Sink Meta Prior.}
Finally, SurvivalPFN can use a meta-prior that mixes multiple complete survival priors. Let $P_1,\ldots,P_M$ denote child prior generators and let $w_1,\ldots,w_M$ be nonnegative mixture weights with $\sum_m w_m=1$. The kitchen-sink prior samples
\[
    M^\ast \ \sim  \ \operatorname{Categorical}(w_1,\ldots,w_M),
    \qquad
    \data^{tr}_{\theta} \ \sim \ P_{M^\ast}.
\]
Equivalently, the induced prior over synthetic datasets is
\[
    P_{\mathrm{sink}}(\data) \ = \ \sum_{m=1}^{M} w_m P_m(\data).
\]
This mixture increases prior diversity by combining direct table-output mechanisms, smooth monotone distributional mechanisms, and positive-time mixture mechanisms. In the current default configuration, the kitchen-sink prior places most mass on the direct table-output and monotone survival-distribution priors, while the exact mixture weights remain configurable.

\paragraph{Diagnostics for Synthetic DGPs.}
We further examine the synthetic DGPs induced by each prior family in
Figures~\ref{fig:prior-quality-comparison} and
\ref{fig:prior-time-distribution-comparison}. These diagnostics are designed to
check the two desiderata of the prior: it should remain close to the
identifiable independent-censoring regime, while still covering a broad range of
survival-data regimes. For each prior family, we sample 500 synthetic datasets,
each with 1{,}024 observations, and compute dataset-level and curve-level
summaries.
\begin{figure}[t]
    \centering

    \begin{subfigure}[t]{0.49\linewidth}
        \centering
        \includegraphics[width=\linewidth]{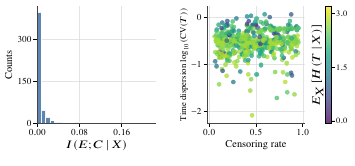}
        \caption{Naive prior}
        \label{fig:prior-quality-naive}
    \end{subfigure}
    \hfill
    \begin{subfigure}[t]{0.49\linewidth}
        \centering
        \includegraphics[width=\linewidth]{figs/survival_dist_prior_quality.pdf}
        \caption{Survival-distribution prior}
        \label{fig:prior-quality-survival-dist}
    \end{subfigure}

    \vspace{0.4em}

    \begin{subfigure}[t]{0.49\linewidth}
        \centering
        \includegraphics[width=\linewidth]{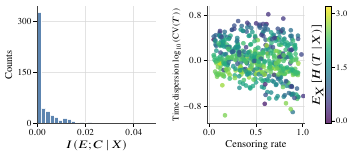}
        \caption{Mixture prior}
        \label{fig:prior-quality-mix-model}
    \end{subfigure}
    \hfill
    \begin{subfigure}[t]{0.49\linewidth}
        \centering
        \includegraphics[width=\linewidth]{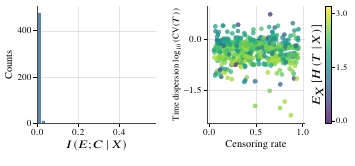}
        \caption{Kitchen-sink prior}
        \label{fig:prior-quality-kitchen-sink}
    \end{subfigure}
    \caption{
    Prior-quality diagnostics across four synthetic prior families, with 500 sampled datasets per family. 
    Each panel summarizes the induced dependence between event and censoring times,
    censoring-rate coverage, time-scale dispersion, and conditional event-time uncertainty.
    }
    \label{fig:prior-quality-comparison}
\end{figure}

\begin{figure}[ht]
    \centering

    \begin{subfigure}[t]{\linewidth}
        \centering
        \includegraphics[width=\linewidth]{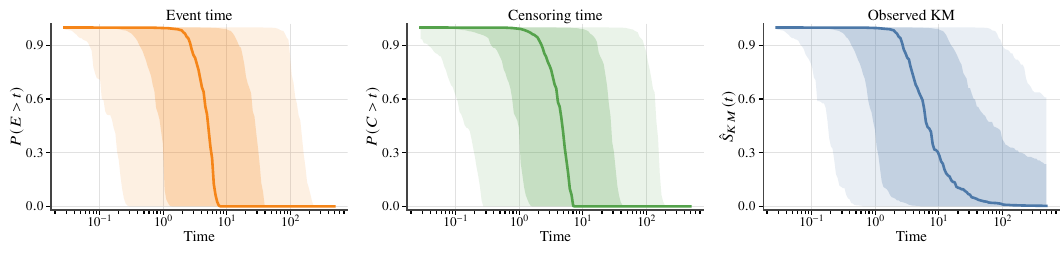}
        \vspace{-2em}
        \caption{Naive prior}
        \label{fig:prior-time-dist-naive}
    \end{subfigure}

    \vspace{0.5em}

    \begin{subfigure}[t]{\linewidth}
        \centering
        \includegraphics[width=\linewidth]{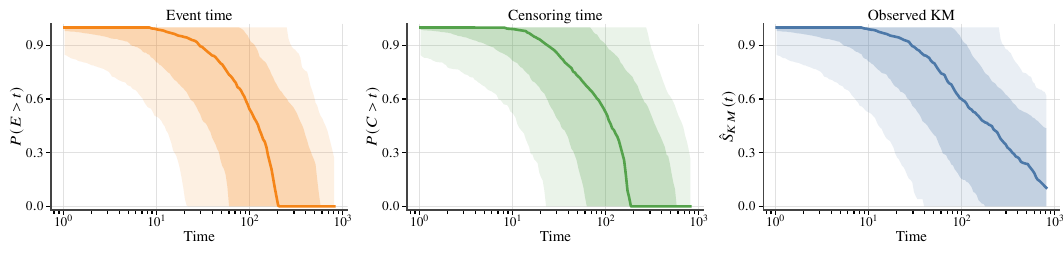}
        \vspace{-2em}
        \caption{Survival-distribution prior}
        \label{fig:prior-time-dist-survival-dist}
    \end{subfigure}

    \vspace{0.5em}

    \begin{subfigure}[t]{\linewidth}
        \centering
        \includegraphics[width=\linewidth]{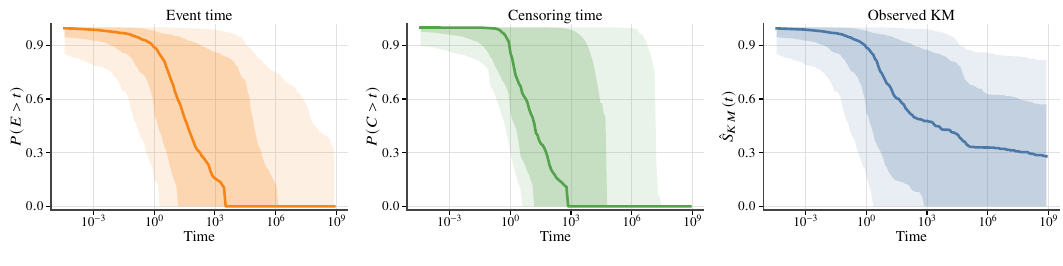}
        \vspace{-2em}
        \caption{Mixture-model prior}
        \label{fig:prior-time-dist-mix-model}
    \end{subfigure}

    \vspace{0.5em}

    \begin{subfigure}[t]{\linewidth}
        \centering
        \includegraphics[width=\linewidth]{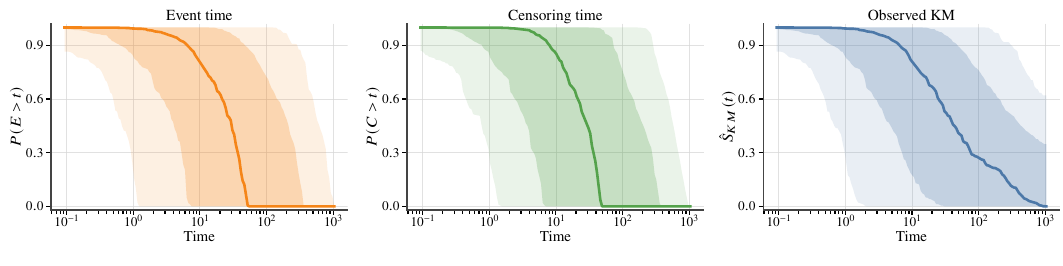}
        \vspace{-2em}
        \caption{Kitchen-sink prior}
        \label{fig:prior-time-dist-kitchen-sink}
    \end{subfigure}

    \caption{
    \textbf{Empirical distributions across 500 synthetic datasets.}
    Each panel shows the induced marginal event-time survival curve $P(E > t)$,
    censoring-time survival curve $P(C > t)$, and observed Kaplan-Meier curve
    $\widehat{S}_{\mathrm{KM}}(t)$.
    Solid lines denote the pointwise median curve across generated datasets. Dark shaded bands denote the interquartile range (25th-75th percentiles), and light shaded bands denote the 10th-90th percentile range.
    }
    \label{fig:prior-time-distribution-comparison}
\end{figure}

Figure~\ref{fig:prior-quality-comparison} reports scalar diagnostics for each generated dataset. 
The left subpanel estimates the conditional mutual information $I(E; C \mid X)$ between the latent event time $E$ and censoring time $C$ after conditioning on covariates $X$. 
Since our simulator is constructed under conditional independence, values concentrated near zero provide an empirical sanity check that the generated censoring process does not introduce substantial residual event-censoring dependence beyond $X$. 
The right subpanel summarizes dataset diversity: each point is one generated dataset, with the x-axis showing the censoring rate, the y-axis showing the observed-time dispersion $\log_{10}(\mathrm{CV}(T))$, where $\mathrm{CV}(T)=\mathrm{std}(T)/|\mathrm{mean}(T)|$ is the coefficient of variation of the observed times, and the color showing the conditional observed-time entropy $\mathbb{E}_{X}[H(T \mid X)]$.
Thus, this panel visualizes how broadly each prior covers censoring rates, relative time-scale variation, and residual outcome uncertainty after conditioning on covariates.

Several trends are apparent from these plots. 
First, across all four prior families, the estimated conditional mutual information $I(E;C\mid X)$ is highly concentrated near zero, suggesting that the generated datasets largely remain within the intended conditional independent censoring regime. 
This provides an empirical sanity check that the prior does not typically generate strongly informative censoring structures.
Second, the scatter plots show that the priors cover a broad range of right-censored survival regimes. 
The generated datasets span nearly the full
range of censoring rates, from lightly censored to heavily censored settings.
They also cover different observed-time dispersion regimes, and different levels of residual conditional uncertainty, as measured by $\mathbb{E}_{X}[H(T\mid X)]$.

Figure~\ref{fig:prior-time-distribution-comparison} complements these scalar summaries with curve-level diagnostics. 
For each generated dataset, we compute the empirical latent event-time survival curve $P(E > t)$, the empirical latent censoring-time survival curve $P(C > t)$, and the Kaplan-Meier estimate $\widehat{S}_{\mathrm{KM}}(t)$ from the observed pairs $(T,\Delta)$. 
The solid line denotes the pointwise median curve across generated datasets, while the dark and light bands denote the pointwise 25th-75th and 10th-90th percentile ranges, respectively. 
These curves show that the priors induce a wide range of event-time, censoring-time, and observed survival shapes, including different
time scales, tail behaviors, and censoring-distorted Kaplan-Meier patterns.
Together, Figures~\ref{fig:prior-quality-comparison} and
\ref{fig:prior-time-distribution-comparison} show that the proposed prior families generate diverse survival tasks while maintaining the identifiable right-censoring structure required by our theoretical framework.
These curve-level summaries further confirm that the diversity of DGPs. 
The empirical distributions vary substantially across generated datasets, with different decay rates, time scales, and tail behaviors.

\subsection{Architecture Details}
\label{app:model_details}

SurvivalPFN uses the same PFN-style transformer architecture as TabDPT~\citep{ma2024tabdpt} and CausalPFN~\citep{balazadeh2025causalpfn}, with only task-specific changes to the token contents and output interpretation. Each context row $(x_i,t_i,\delta_i)\in\data^{tr}_{\theta}$ is represented as a single token by combining embeddings of the covariates $x_i$, observed time $t_i$, and event indicator $\delta_i$. Each query row is represented by embeddings of the query covariate $x^\ast_\theta$ and the query indicator $\widetilde{\delta}^\ast$. We use linear embedding layers and omit positional encodings, so the context is treated as a set rather than an ordered sequence.

All context and query tokens are passed through a 20-layer transformer encoder with hidden dimension 384, RMS query-key normalization, and parallel SwiGLU feed-forward blocks. The attention mask follows the standard PFN design~\citep{hollmann2022tabpfn, hollmann2025accurate}: context tokens attend to one another, while query tokens attend only to the context. This masking prevents information leakage across queries and allows the model to process many query instances in parallel. The output representation of each query token is projected to $L=1024$ logits, and a softmax produces a discretized PPD over the transformed time bins:
\[
    q_\omega(\cdot\mid x^\ast_\theta,\widetilde{\delta}^\ast,\data^{tr}_{\theta})
    \ = \ 
    \left[
    q_{\omega,\ell}(x^\ast_\theta,\widetilde{\delta}^\ast,\data^{tr}_{\theta})
    \right]_{\ell=1}^{L}.
\]
Depending on $\widetilde{\delta}^\ast$, this distribution is interpreted as an approximation to either the PPD for event time or the PPD for censoring time. The corresponding PPSD or PPCD is obtained by summing the predicted tail probability mass across bins.

The full model has approximately 20M parameters and is trained in two stages: (i) a predictive phase that follows standard predictive PFN training from \citet{ma2024tabdpt}, and (ii) a survival phase that optimizes the survival likelihood or cross-entropy loss. 
We use AdamW~\citep{kingma2014adam} with warmup and cosine annealing in the predictive phase, and switch to the schedule-free optimizer~\citep{defazio2024road} in the survival phase. 
The maximum context length is 16K in the first phase and 2{,}048 in the second. The predictive phase is trained on four A100 GPUs for up to one week, followed by two and half days of survival-phase training on one H100 GPU.

Training is parallelized over both synthetic context and query tokens. At each gradient step, we sample $B_\theta$ independent DGPs from the prior, generate one context table for each DGP, and draw $B_q$ query rows per datasets. These $B_\theta B_q$ query predictions are computed in a single batched forward pass, and the final loss is averaged over all DGP-query pairs. This batching strategy is identical in spirit to the parallel training procedure used in CausalPFN; see Algorithm 1 in~\citet{balazadeh2025causalpfn} for the full procedure.

\subsection{Monotone Time Transformations}
\label{app:monotone_transformation}

Survival times are nonnegative and often highly skewed. Directly discretizing raw time can therefore allocate too many bins to sparse tail regions or make the learned distribution sensitive to dataset-specific time scales. 
To address this, SurvivalPFN applies a context-fitted monotone transformation before discretization. 

For each dataset $\data=\{(x_i,t_i,\delta_i)\}_{i=1}^{N}$,
let $g: \mathbb{R}_{+}\to\mathcal{Z}$ denote the dataset-specific transformation fitted on the observed times for the context tokens $\{t_i\}_{i=1}^{N_c}$, where $N_c$ is the size of context tokens. 
The observed times and query targets for context tokens are mapped into model space as
\begin{align*}
    z_i \ &= \ g (t_i),\\
    z_e^\ast \ &= \ g(e^\ast), \\
    z_c^\ast \ &= \ g(c^\ast).
\end{align*}
The histogram likelihood \Eqref{eq:survivalpfn_nll_loss} or cross-entropy loss in \Eqref{eq:ce_loss} is then evaluated in this transformed space. 

To make prediction, model-space bin edges are mapped back to raw time through $g^{-1}$, so that the final PPSD/PPCD is reported on the original time scale. In all cases, $g$ is monotone increasing, preserving the temporal ordering of events and censoring times.

\paragraph{\texttt{lognormal2normal} Transformation.}
The \texttt{lognormal2normal} transformation uses a simple parametric normalization motivated by the positivity and right-skewness of survival times. For each dataset, we first calculate the mean $m$ and variance $s^2$ on the context tokens.
We fit a lognormal distribution whose raw-time mean and standard deviation match $m$ and $s^2$. If $T=\exp(\mu+\sigma Z)$ with $Z\sim\mathcal{N}(0,1)$, then
\begin{align*}
    \sigma^2 \ &= \ \log\!\left(1+\frac{s^2}{m^2}\right),\\
    \mu \ &= \ \log(m)-\frac{1}{2}\sigma^2 .
\end{align*}
The forward transformation maps a raw time $t>0$ to its fitted normal coordinate:
\[
    g_\mathrm{LN}(t) \ = \ \frac{\log t-\mu}{\sigma}.
\]
The inverse transformation is
\[
    g_\mathrm{LN}^{-1}(z) \ = \ \exp(\mu+\sigma z).
\]
Thus, fixed-width bins in model space correspond to adaptive, nonuniform bins in raw time. This transformation is smooth, bijective on $\mathbb{R}_{+}$, preserves time ordering, and can extrapolate beyond the largest observed time. Its main limitation is that it imposes a lognormal shape on the observed-time distribution; this parametric normalization may not allocate resolution optimally when the parametric form is incorrect.

\paragraph{\texttt{time2quantile} Transformation.}
The \texttt{time2quantile} transformation is a nonparametric, context-adaptive alternative. It maps raw time into an empirical quantile coordinate in $[0,1]$. For all the context tokens in the dataset, sort the observed times and collapse duplicates to obtain unique knots
\[
    0 \ = \ a_0  \ < \ a_1 \ < \ \cdots \ <\  a_K\ =\ \max_i t_i .
\]
Each knot $a_j$ is assigned its right-continuous empirical CDF value
\[
    q_j \ = \ \widehat F(a_j) \ = \
    \frac{1}{N_c}\sum_{i=1}^{N_c}\mathbbm{1}\{t_i\le a_j\},
    \qquad
    q_0 \ =\ 0,\quad q_K\ =\ 1.
\]
The forward transformation is the piecewise-linear interpolation of these knots:
\[
    g_\mathrm{Q}(t) \ = \ q_j +
    \frac{t-a_j}{a_{j+1}-a_j}
    (q_{j+1}-q_j),
    \qquad
    a_j \ \leq \ t\ \leq  \ a_{j+1}.
\]
Values above the largest time are mapped to $1$. The inverse transformation swaps the axes and linearly interpolates from quantile space back to raw time:
\[
    g_\mathrm{Q}^{-1}(q) \ = \
    a_j
    +
    \frac{q-q_j}{q_{j+1}-q_j}
    (a_{j+1}-a_j),
    \qquad
    q_j \ \leq \ q \ \leq \ q_{j+1}.
\]
For this transformation, the model-space range is fixed to $[0,1]$. Uniform bins in quantile space become adaptive raw-time bins: regions with many observed times receive finer resolution, while sparse regions receive wider bins.

The \texttt{time2quantile} transformation makes no parametric assumption on the time distribution and is robust to changes in time units or monotone rescalings of raw time. However, because it is fitted from the empirical distribution of $\{t_i\}_{i=1}^{N_c}$, it is context-local and cannot resolve the tail shape beyond $\max_i t_i$; all larger times are mapped to quantile $1$, with probability mass beyond this point represented only by the final residual histogram bin.

\paragraph{Summary.}
The two transformations reflect complementary design choices. The \texttt{lognormal2normal} transformation provides a smooth positive-time coordinate with parametric tail extrapolation, while \texttt{time2quantile} provides a fully context-adaptive coordinate that normalizes all tasks to the common interval $[0,1]$. Both transformations preserve time ordering and allow SurvivalPFN to allocate discretization resolution more effectively than raw-time binning.

\subsection{Training Objective}
\label{app:train_objective}

SurvivalPFN is trained to predict a discretized PPD over transformed time. 
For a query with indicator $\widetilde{\delta}^{\ast}$, define the latent supervised target
\[
    r^\ast_\theta(\widetilde{\delta}^{\ast})
    \ = \
    \widetilde{\delta}^{\ast} e^\ast_\theta
    +
    \bigl(1-\widetilde{\delta}^{\ast}\bigr)c^\ast_\theta .
\]
Thus, $\widetilde{\delta}^{\ast}=1$ asks for the latent event time, whose posterior predictive tail gives the PPSD, while $\widetilde{\delta}^{\ast}=0$ asks for the latent censoring time, corresponding to the PPCD.

Let $g_{\data^{tr}_{\theta}}:\mathbb{R}_{+}\to\mathcal{Z}$ be the monotone transformation fitted from the observed context times in $\data^{tr}_{\theta}$, and let
$\{\mathcal I_\ell\}_{\ell=1}^{L}$ denote the ordered bins in transformed-time space. 
Define the bin-index map
\[
    \kappa_{\data^{tr}_{\theta}}(r)
    \ = \
    \ell
    \quad
    \text{if}
    \quad
    g_{\data^{tr}_{\theta}}(r)\in \mathcal I_\ell ,
\]
with boundary clipping handled by the same convention used in implementation. 
Equivalently, define the one-hot target
\[
    b_{\ell,\data^{tr}_{\theta}}(r)
    \ = \
    \mathbb{I}\!\left\{
        \kappa_{\data^{tr}_{\theta}}(r)=\ell
    \right\},
    \qquad
    \ell=1,\ldots,L .
\]

\paragraph{Likelihood loss.}
The main SurvivalPFN checkpoint is trained with the one-hot discrete negative log-likelihood over transformed-time bins:
\[
    \mathrm{NLL}\!\left(
        r^\ast_\theta
        \,\middle\|\,
        q_\omega
    \right)
    \ = \
    -
    \log
    q_{\omega,\kappa_{\data^{tr}_{\theta}}(r^\ast_\theta)}
    \!\left(
        x^\ast_\theta,
        \widetilde{\delta}^{\ast},
        \data^{tr}_{\theta}
    \right).
\]
Equivalently,
\[
    \mathrm{NLL}\!\left(
        r^\ast_\theta
        \,\middle\|\,
        q_\omega
    \right)
    \ = \
    -
    \sum_{\ell=1}^{L}
    b_{\ell,\data^{tr}_{\theta}}(r^\ast_\theta)
    \log
    q_{\omega,\ell}
    \!\left(
        x^\ast_\theta,
        \widetilde{\delta}^{\ast},
        \data^{tr}_{\theta}
    \right).
\]
The corresponding population objective is
\begin{align}
\label{eq:app_survivalpfn_nll_loss}
    \mathcal{L}_{\mathrm{NLL}}(\omega)
    \ = \
    \mathbb E_{\theta\sim\pi(\cdot)}
    \, \mathbb E_{\data^{tr}_{\theta},\,x^\ast_\theta,\,e^\ast_\theta,\,c^\ast_\theta,\,\widetilde{\delta}^{\ast}}
    \left[
    -
    \log
    q_{\omega,\kappa_{\data^{tr}_{\theta}}\!\left(r^\ast_\theta(\widetilde{\delta}^{\ast})\right)}
    \!\left(
        x^\ast_\theta,
        \widetilde{\delta}^{\ast},
        \data^{tr}_{\theta}
    \right)
    \right].
\end{align}
This is the objective used by the best validation checkpoint reported in the main results.

\paragraph{Smoothed cross-entropy variant.}
We also consider a smoothed cross-entropy loss, which is used in TabDPT~\citep{ma2024tabdpt} and CausalPFN~\citep{balazadeh2025causalpfn}. 
Instead of assigning all mass to one bin, the target time is converted into a narrow histogram in transformed-time space:
\[
    a_{\ell,\data^{tr}_{\theta}}^{(\sigma)}(r)
    \ = \
    \int_{\mathcal I_\ell}
    \alpha_\sigma
    \!\left(
        z;\,
        g_{\data^{tr}_{\theta}}(r)
    \right)
    dz,
    \qquad
    \ell=1,\ldots,L,
\]
where $\alpha_\sigma$ is a narrow smoothing density, implemented as a Gaussian centered at $g_{\data^{tr}_{\theta}}(r)$ with a predefined variance. 
The smoothed cross-entropy loss is
\begin{align}
\label{eq:ce_loss}
    \mathrm{SCE}_{\sigma}\!\left(
        r^\ast_\theta
        \,\middle\|\,
        q_\omega
    \right)
    \ = \
    -
    \sum_{\ell=1}^{L}
    a_{\ell,\data^{tr}_{\theta}}^{(\sigma)}(r^\ast_\theta)
    \log
    q_{\omega,\ell}
    \!\left(
        x^\ast_\theta,
        \widetilde{\delta}^{\ast},
        \data^{tr}_{\theta}
    \right).
\end{align}
As $\sigma\to 0$, the smoothed target
$a_{\ell,\data^{tr}_{\theta}}^{(\sigma)}(r)$ reduces to the one-hot target
$b_{\ell,\data^{tr}_{\theta}}(r)$, and the smoothed cross-entropy recovers the discrete NLL. 
In our experiments, this loss is treated as an alternative training option and evaluated in the ablation study.

\paragraph{Note.}
This prior-data likelihood differs from the standard right-censored survival likelihood in Section~\ref{sec:background}. 
The standard observed-data likelihood trains on partially observed query outcomes
$(t^\ast,\delta^\ast)$ and must account for censored queries through survival-tail probabilities. 
In contrast, SurvivalPFN training uses simulator-provided latent query times
$e^\ast_\theta$ or $c^\ast_\theta$ as supervision.

\subsection{Inference Procedure}
\label{app:inference_details}
At inference time, SurvivalPFN takes an observed survival dataset $\data$ as context and one or more query covariates $x^\ast$. For survival prediction, we set the query indicator to $\widetilde{\delta}^\ast=1$, so that the model outputs a discretized posterior predictive event-time distribution
\[    
q_\omega\!\left(\cdot \,\middle\vert\, x^\ast,\widetilde{\delta}^\ast=1,\data\right)    
\ = \ \left[    
q_{\omega,\ell}(x^\ast,\widetilde{\delta}^\ast=1,\data)    
\right]_{\ell=1}^{L}.
\]
This requires only a single forward pass and does not involve dataset-specific gradient updates, posterior sampling, or hyperparameter tuning.
Let $g_{\data}$ denote the monotone time transformation fitted from the observed times in $\data$, and let $\{\mathcal I_\ell\}_{\ell=1}^{L}$ denote the discretized bins in the transformed time space. 
The predicted posterior predictive survival distribution (PPSD) is obtained by summing the probability mass assigned to bins whose raw-time support lies above $t$:
\[    
\widehat S_\omega(t\mid x^\ast,\data) \ = \ \sum_{\ell:\, g_{\data}^{-1}(\mathcal I_\ell) > t}    
q_{\omega,\ell}\!\left(x^\ast,\widetilde{\delta}^\ast=1,\data\right),
\]
which is implemented as a cumulative tail sum over the discretized time bins. Similarly, setting $\widetilde{\delta}^\ast=0$ yields the PPCD (although we generally do not care about PPCD during inference).

\section{Benchmarking Details}

\subsection{Baseline Model Details}
\label{app:model}

In this study, we compare SurvivalPFN against a board range of 20 representative survival analysis methods spanning classical statistical models, tree-based ensembles, discrete-time neural survival models, continuous-time neural survival models, and methods using TFMs. Below, we briefly summarize each baseline, including its core modeling idea and the implementation details used in our benchmark. A side-to-side overview of their methodological properties is provided in Table~\ref{tab:survival_model_comparison}.

For each column in the table, \emph{continuous-time} indicates whether the method explicitly parameterizes a smooth survival distribution over continuous time, without relying on post-hoc interpolation over a discrete time grid. \emph{(Semi-)parametric} indicates whether the method specifies the hazard, density, or survival function through an explicit parametric or semi-parametric form. 
The \emph{proportional hazards (PH)} assumption indicates that covariates act multiplicatively on a shared baseline hazard, so that hazard ratios between individuals are constant over time. 
Finally, the \emph{ensemble/mixture} column indicates whether the method combines a finite collection of base learners, components, or experts, such as trees in an ensemble or distributions in a finite mixture model.

In particular, SurvivalMDN is marked parametric because it uses a finite mixture of Gaussians, and DSM is marked parametric because it uses a finite mixture of Weibulls; only an infinite-mixture construction would be nonparametric. 

\clearpage
\begingroup
\begin{landscape}
\thispagestyle{plain}
\begin{table}[p]
\centering
\setlength{\tabcolsep}{3.2pt}
\begin{threeparttable}
\caption{Qualitative comparison of survival models considered in our benchmark. Columns indicate whether a method explicitly parameterizes a continuous-time (cont. time) survival distribution, imposes a (semi-)parametric form, assumes proportional hazards (PH), uses an ensemble or finite-mixture construction, or belongs to the tabular foundation model (TMF) family.}
\label{tab:survival_model_comparison}
\begin{tabular}{@{}>{\raggedright\arraybackslash}p{0.17\linewidth}*{5}{>{\centering\arraybackslash}p{0.074\linewidth}}>{\raggedright\arraybackslash}p{0.38\linewidth}@{}}
\toprule
\textbf{Model} &
\makecell{\textbf{Cont.}\\\textbf{time}} &
\makecell{\textbf{(Semi-)}\\\textbf{parametric}} &
\makecell{\textbf{PH}\\\textbf{assump.}} &
\makecell{\textbf{Ensemble/}\\\textbf{mixture}} &
\textbf{TFM} &
\textbf{Distinguishing feature} \\
\midrule
\multicolumn{7}{@{}l}{\textit{Prior-fitted and tabular foundation models}} \\
SurvivalPFN                 & \benchxmark & \benchxmark & \benchxmark & \benchxmark & \benchcmark & Novel in-context Bayesian survival estimator. \\
StaticSurvivalTFM           & \benchxmark & \benchxmark & \benchxmark & \benchxmark & \benchcmark & Handling right-censoring with tabular foundation models. \\
\midrule
\multicolumn{7}{@{}l}{\textit{Classical survival models}} \\
CoxPH                       & \benchxmark & \benchcmark & \benchcmark & \benchxmark & \benchxmark & Semi-parametric Cox proportional hazards model. \\
CoxNet                      & \benchxmark & \benchcmark & \benchcmark & \benchxmark & \benchxmark & Regularized Cox proportional hazards model. \\
cSVR                        & \benchxmark & \benchxmark & \benchxmark & \benchxmark & \benchxmark & Censored support-vector regression baseline. \\
\midrule
\multicolumn{7}{@{}l}{\textit{Tree-based models}} \\
GB                          & \benchxmark & \benchcmark & \benchcmark & \benchcmark & \benchxmark & Gradient-boosted Cox proportional hazards model. \\
CWGB                        & \benchxmark & \benchcmark & \benchcmark & \benchcmark & \benchxmark & Censoring-weighted gradient-boosted Cox model. \\
RSF                         & \benchxmark & \benchxmark & \benchxmark & \benchcmark & \benchxmark & Random survival forest for nonlinear effects and interactions. \\
\midrule
\multicolumn{7}{@{}l}{\textit{Neural discrete-time models}} \\
DeepHit                     & \benchxmark & \benchxmark & \benchxmark & \benchxmark & \benchxmark & Neural discrete-time model emphasizing discrimination. \\
DeepSurv                    & \benchxmark & \benchcmark & \benchcmark & \benchxmark & \benchxmark & Neural Cox proportional hazards model. \\
MTLR                        & \benchxmark & \benchxmark & \benchxmark & \benchxmark & \benchxmark & Multi-task logistic regression for discrete-time survival. \\
Nnet-survival               & \benchxmark & \benchxmark & \benchxmark & \benchxmark & \benchxmark & Discrete-time neural survival model. \\
CoxTime                     & \benchxmark & \benchcmark & \benchxmark & \benchxmark & \benchxmark & Neural Cox model with time-varying effects. \\
IWSG                        & \benchxmark & \benchxmark & \benchxmark & \benchxmark & \benchxmark & Explicitly models the censoring mechanism. \\
CQRNN                       & \benchxmark & \benchxmark & \benchxmark & \benchxmark & \benchxmark & Quantile-regression survival baseline. \\
BNN-MTLR                    & \benchxmark & \benchxmark & \benchxmark & \benchxmark & \benchxmark & Bayesian neural extension of MTLR. \\
\midrule
\multicolumn{7}{@{}l}{\textit{Neural continuous-time models}} \\
DSM                         & \benchcmark & \benchcmark & \benchxmark & \benchcmark & \benchxmark & Finite parametric (Weibull) mixture components. \\
SumoNet                     & \benchcmark & \benchxmark & \benchxmark & \benchxmark & \benchxmark & Continuous survival-function via automatic differentiation. \\
SurvivalMDN                 & \benchcmark & \benchcmark & \benchxmark & \benchcmark & \benchxmark & Finite parametric (Gaussian) mixture components. \\
DeepAFT--Weibull            & \benchcmark & \benchcmark & \benchcmark & \benchxmark & \benchxmark & Continuous Weibull accelerated failure time model. \\
DeepAFT--Log-logistic       & \benchcmark & \benchcmark & \benchxmark & \benchxmark & \benchxmark & Continuous log-logistic accelerated failure time model. \\
\bottomrule
\end{tabular}
\end{threeparttable}
\end{table}
\end{landscape}
\endgroup
\clearpage

\begin{enumerate}
    \item \textbf{SurvivalPFN}. SurvivalPFN is the model we developed in this paper. See description in Section~\ref{sec:method} and Appendix~\ref{app:survivalpfn_details}.
    
    \item \textbf{StaticSurvivalTFM}~\citep{kim2026tabular}. 
    StaticSurvivalTFM is a framework that converts any binary classifier into a survival predictor; we provide methodological details in Appendix~\ref{app:concurrent_works}.
    For a fair comparison with SurvivalPFN, we use TabDPT~\citep{ma2024tabdpt} as the backbone binary classifier in the main benchmark.
    We further study the effect of replacing this backbone with MITRA~\citep{zhang2025mitra} (the reported best version in~\citet{kim2026tabular}), with results reported in Appendix~\ref{app:rq4}.
        
    \item \textbf{Cox proportional hazards model (CoxPH)}~\citep{cox1972regression}.
    CoxPH is the classical semi-parametric proportional-hazards model, estimating covariate effects through the Cox partial likelihood while leaving the baseline hazard unspecified. We used the \texttt{scikit-survival} implementation~\citep{polsterl2020scikit}, which produces risk scores and can recover survival curves using an estimated baseline hazard~\citep{breslow1975analysis}.

    \item \textbf{Elastic-net Cox proportional hazards model (CoxNet)}~\citep{simon2011regularization}.
    CoxNet regularizes the Cox partial likelihood with an elastic-net penalty, making Cox-style modeling more stable in high-dimensional settings. We used the \texttt{scikit-survival} implementation~\citep{polsterl2020scikit}, with the same baseline hazard estimator as CoxPH.

    \item \textbf{Censored Support Vector Regression (cSVR)}~\citep{polsterl2015fast}.
    cSVR extends support-vector regression to right-censored outcomes by treating censored observations through inequality constraints or ranking-style losses. Specifically, we use the \texttt{FastSurvivalSVM} method in \texttt{scikit-survival}~\citep{polsterl2020scikit}. Because cSVR only outputs a scalar regression time rather than a full survival distribution, it cannot produce valid values for evaluation metrics that require an entire survival curve or event-time distribution.

    \item \textbf{Gradient-boosted Cox model (GB)}~\citep{ridgeway1999state}.
    GB fits an additive risk model by gradient boosting under a Cox partial-likelihood objective. We used the \texttt{scikit-survival} implementation~\citep{polsterl2020scikit}, which yields a boosted risk score and survival estimates through the fitted baseline hazard as CoxPH.

    \item \textbf{Component-wise gradient-boosted Cox model (CWGB)}~\citep{hothorn2006survival}.
    CWGB is a component-wise variant of gradient-boosted Cox regression, where weak learners update individual covariate components under a Cox-style objective. We used the \texttt{scikit-survival} implementation~\citep{polsterl2020scikit}.

    \item \textbf{Random Survival Forests (RSF)}~\citep{ishwaran2008random}.
    RSF is an ensemble of survival trees that partitions the feature space using survival-specific split criteria and averages terminal-node survival estimates across trees. We use the \texttt{scikit-survival} implementation~\citep{polsterl2020scikit}.

    \item \textbf{DeepHit} \citep{lee2018deephit}.
    DeepHit is a discrete-time neural survival model that directly parameterizes the probability mass function over event-time bins. We used the \texttt{pycox} implementation; its objective combines a likelihood term with a C-index-like ranking term, explicitly encouraging discrimination.

    \item \textbf{DeepSurv} \citep{katzman2018deepsurv}.
    DeepSurv replaces the linear predictor in CoxPH with a neural network while retaining the Cox partial-likelihood objective and proportional-hazards structure. We reimplemented it following the public PyTorch implementation (\url{https://github.com/czifan/DeepSurv.pytorch}) and add baseline hazard estimation as in CoxPH.

    \item \textbf{Multi-task Logistic Regression (MTLR)}~\citep{yu2011learning,fotso2018deep}.
    MTLR parameterizes the survival distribution over a discrete time grid using a sequence of dependent logistic regressors. We reimplemented MTLR following the public PyTorch implementation (\url{https://github.com/mkazmier/torchmtlr}).

    \item \textbf{Nnet-survival} \citep{biganzoli1998feed,gensheimer2019scalable}.
    Nnet-survival models discrete-time conditional hazards with a neural network and trains by a binary cross-entropy-style survival likelihood over time intervals. We used the \texttt{pycox} implementation, converting predicted discrete hazards into survival curves for distributional evaluation.

    \item \textbf{CoxTime} \citep{kvamme2019time}.
    CoxTime generalizes neural Cox regression by allowing the relative risk to depend on both covariates and time, thereby relaxing the proportional-hazards assumption. We used the \texttt{pycox} implementation and recovered survival curves from the learned time-dependent risk function and the same baseline hazard estimator as CoxPH.

    \item \textbf{Inverse-Weighted Survival Games (IWSG)} \citep{han2021inverse}.
    IWSG explicitly models both the failure-time and censoring distributions through an inverse-probability-of-censoring-weighted (IPCW) game objective. We reimplemented it from the official IWSG codebase (\url{https://github.com/rajesh-lab/Inverse-Weighted-Survival-Games}).

    \item \textbf{Censored Quantile Regression Neural Network (CQRNN)} \citep{pearce2022censored}.
    CQRNN directly predicts event-time quantiles under censoring, providing a distribution-free way to represent time-to-event uncertainty. We reimplemented it following the official CQRNN codebase (\url{https://github.com/TeaPearce/Censored_Quantile_Regression_NN}) and converted the predicted quantile function into a monotone survival curve when distributional metrics (IBS and D-calibration) were required.

    \item \textbf{Bayesian Neural Network Multi-task Logistic Regression (BNN-MTLR)}~\citep{qi2023using}.
    BNN-MTLR extends MTLR with Bayesian neural-network uncertainty, producing a PPD over discrete survival curves. We reimplemented it from the official BNN-MTLR codebase (\url{https://github.com/shi-ang/BNN-ISD}).

    \item \textbf{Deep Survival Machines (DSM)} \citep{nagpal2021dsm}.
    DSM parameterizes the event-time distribution as a finite mixture of Weibull components, with neural networks producing mixture weights and component parameters. We reimplemented DSM using the code in \texttt{auton-survival} package~\citep{nagpal2022auton}.

    \item \textbf{Survival Monotonic Network (SuMoNet)}~\citep{rindt2022survival}.
    SuMoNet models continuous-time survival distributions with monotonic neural networks, using automatic differentiation to obtain valid densities from the learned cumulative distribution function. We reimplemented it following the official SuMoNet codebase (\url{https://github.com/MrHuff/Sumo-Net}).


    \item \textbf{Survival Mixture Density Network (SurvivalMDN)} \citep{han2022survival}.
    SurvivalMDN models the event-time distribution using a finite mixture density network, providing a flexible but still finite-dimensional parametric mixture representation. We reimplemented it from the official SurvivalMDN codebase (\url{https://github.com/XintianHan/Survival-MDN}).

    \item \textbf{Neural Weibull accelerated failure-time model (DeepAFT-Weibull)} \citep{norman2024deepaft}.
    DeepAFT-Weibull uses a neural network to parameterize a Weibull accelerated failure-time distribution for right-censored data. We reimplemented the model following the paper.

    \item \textbf{Neural log-logistic accelerated failure-time model (DeepAFT-Loglogistic)} \citep{norman2024deepaft}.
    DeepAFT-Loglogistic similarly uses a neural network to parameterize a log-logistic accelerated failure-time distribution. We reimplemented the model following the paper.
\end{enumerate}

\subsection{Unified Time Grid for Consistent Evaluation}

All methods in our benchmark are evaluated through a common prediction object. 
After fitting, each model is converted into a survival-probability matrix
\begin{align*}
    \widehat{\mathbf{S}} \ = \
    \left[\widehat{S}_i(t_j)\right]_
    {i=1,\ldots,n_{\mathrm{test}};\,j=1,\ldots,m},
    \qquad
    \mathcal{G} \ =\ \{t_1 < \cdots < t_m\},
\end{align*}
where $\widehat{S}_i(t_j)$ denotes the predicted survival probability for test individual $i$ at time $t_j$. 
All metrics are then computed from $(\mathcal{G}, \widehat{\mathbf{S}})$. 
Thus, differences across methods enter only through how their native time representation is constructed before prediction.

To avoid evaluation leakage, all time grids are defined using only the data available during training; test-set event times are never used to define the evaluation support.

For models that require a predefined binning for making discrete-time survival function prediction, including MTLR, DeepHit, Nnet-survival, BNN-MTLR and StaticSurvivalTFM, we use an event-time quantile grid. 
Let
\begin{align*}
    \mathcal{T}_E \ = \ \{T_i: \delta_i=1,\ i\in\data^{tr} \}
\end{align*}
denote the uncensored event times in the fitting data. 
When the number of bins is not specified, we set
\begin{align*}
    K\ = \ \left\lceil \sqrt{|\mathcal{T}_E|}\right\rceil,
\end{align*}
and define the discrete support by the unique empirical quantiles
\begin{align*}
    \mathcal{G}_{\mathrm{disc}} \ = \
    \operatorname{unique}
    \left\{
    Q_{\mathcal T_E}\!\left(\frac{k-1}{K-1}\right)
    :\ k=1,\ldots,K
    \right\},
\end{align*}
where $Q_{\mathcal T_E}$ is the empirical quantile function. 
This grid provides the native discretization on which discrete-time models are trained, with each bin contains the exact same number of uncensored instances.

For DeepHit and Nnet-survival, the implementation requries that the first bin location must smaller than the smallest time in the data. Therefore, we additionally shift the first bin location:
\begin{align*}
    b_1 \ \leftarrow  \
    \max\left\{
    \min_{i\in\data^{tr}} \{T_i-\epsilon\},\ 0
    \right\},
    \qquad
    \epsilon=10^{-5}.
\end{align*}

Continuous-time models do not require a training-time discretization. 
Nevertheless, curve-based evaluation still requires a finite query set. 
For these methods, we define the evaluation support from the observed fitting durations:
\begin{align}
    \mathcal G_{\mathrm{cont}} \ = \ 
    \{0\}\cup
    \operatorname{unique}\{T_i:\ i\in\data^{tr}\},
    \end{align}
with duplicate zeros removed if necessary. 
For quantile-output models such as CQRNN, predicted quantile functions are first converted to survival probabilities on the same support before metric computation.

This protocol preserves each model class's natural time parameterization: discrete-time methods learn on event-time quantile bins, whereas continuous-time methods are queried on the empirical training-time support. 

\subsection{Hyperparameter Tuning}
\label{app:hparam_tuning}

We tune hyperparameters only for neural-network baselines whose performance depends on optimizer and architecture choices. Classical \texttt{scikit-survival} models (CoxPH, CoxNet, cSVR, GB, CWGB, and RSF), StaticSurvivalTFM, and SurvivalPFN are kept fixed at their benchmark 
settings.

For each outer split $r$, hyperparameter selection is performed using only the 
training-side data,
\begin{align*}
    \data^{tr + val}_{(r)} \ = \  \data^{tr}_{(r)} \cup \data^{val}_{(r)} ,
\end{align*}
and the test set is never used for either tuning or discretization. For each
model $m$, we sample $R=10$ configurations from its search space
$\Lambda_m$ and estimate their performance by shuffled $F=5$-fold
cross-validation on $\data^{tr+val}_{(r)}$, using the outer
split seed for reproducibility. The selected configuration is 
\begin{align*}
    \lambda_{m,(r)}^\ast \ = \
    \arg\min_{\lambda \in \{\lambda_1,\ldots,\lambda_R\}}
    \frac{1}{F} \sum_{f=1}^{F}
    \mathcal{L}^{val}_{(f)}(\lambda; m),
\end{align*}
where $\mathcal{L}^{val}_{(f)}$ denotes the model-specific objective loss on the validation on fold $f$. 
Configurations that fail to complete training are treated as infeasible and assigned infinite validation loss. 
After selection, model $m$ is refit on $\mathcal D^{tr+val}_{(r)}$ using $\lambda_{m,(r)}^\ast$, and evaluated once on the held-out test split.

The full model-specific search spaces are summarized in Table~\ref{tab:hparam_search}.
All remaining optimization settings are fixed across tuning trials: AdamW
optimizer, batch size $256$, ReLU activations, no normalization layer, early
stopping, and a maximum budget of $10{,}000$ epochs. For models that require
an explicit learning-rate floor, we set
\begin{align*}
    \eta_{\min}\ = \ 10^{-3}\eta,
\end{align*}
where $\eta$ is the tuned learning rate. Model-specific constants not included
in Table~\ref{tab:hparam_search} are fixed throughout tuning.

\begin{table}[t]
\centering
\setlength{\tabcolsep}{1pt}
\begin{threeparttable}
\caption{
Hyperparameter search spaces for tuned neural baselines.
The table defines the shared base space and model-specific extensions.
}
\label{tab:hparam_search}
\begin{tabular}{
@{}
>{\raggedright\arraybackslash}p{0.23\linewidth}
>{\raggedright\arraybackslash}p{0.36\linewidth}
>{\raggedright\arraybackslash}p{0.39\linewidth}
@{}
}
\toprule
\textbf{Models}
&
\textbf{Hyperparameter meanings}
&
\textbf{Search space}
\\
\midrule
\multicolumn{3}{@{}l}{\textit{Base hyperparameters}} \\

\makecell[l]{DeepHit\\ DeepSurv \\ MTLR \\ Nnet-survival \\ CoxTime \\ DeepAFT-Weibull \\ DeepAFT-Loglogistic}
&
\makecell[l]{
\texttt{lr}: learning rate;\\
\texttt{weight\_decay}: weight decay;\\
\texttt{neurons}: hidden architecture;\\
\texttt{dropout}: dropout probability
}
&
\makecell[l]{
\{ \\ 
\texttt{lr}: \(\{10^{-4},10^{-3},10^{-2}\}\); \\
\texttt{weight\_decay}: \(\{10^{-3},10^{-2},10^{-1}\}\);\\
\texttt{neurons}: \(\{[64], [64,32], [64,64,16],\) \\
\qquad \qquad \quad \([32], [32,16], [32,32,16],\) \\
\qquad \qquad \quad \([16], [16,8], [8], []\}\);\\
\texttt{dropout}: \(\{0.0,0.4,0.6\}\) \\
\}
}
\\
\midrule
\multicolumn{3}{@{}l}{\textit{Model-specific hyperparameters}} \\

CQRNN
&
\texttt{n\_quantiles}: number of predicted quantile levels.
&
\(\{\)
Base;
\texttt{n\_quantiles}: \(\{9,19,39\}\)
\(\}\)
\\

SumoNet
&
\texttt{neurons\_alter}: hidden architecture for the censoring branch.
&
\(\{\)
Base;
\texttt{neurons\_alter} = \texttt{neurons}
\(\}\)
\\

IWSG
&
\texttt{neurons\_alter}: hidden architecture for the censoring branch.
&
\(\{\)
Base;
\texttt{neurons\_alter} = \texttt{neurons}
\(\}\)
\\

SurvivalMDN
&
\texttt{n\_mixtures}: number of mixture components.
&
\(\{\)
Base;
\texttt{n\_mixtures}: \(\{3,5,10,50\}\)
\(\}\)
\\

DSM
&
\texttt{n\_mixtures}: number of mixture components.
&
\(\{\)
Base;
\texttt{n\_mixtures}: \(\{3,5,10,50\}\)
\(\}\)
\\

BNN-MTLR
&
\texttt{pi}: prior mixture probability.
&
\(\{\)
Base;
\texttt{pi}: \(\{0.2,0.5,0.8\}\)
\(\}\)
\\

\bottomrule
\end{tabular}
\end{threeparttable}
\end{table}

\subsection{Evaluation Metrics}
\label{app:metrics}

We evaluate the model's performance on the held-out test set $\data^{te}$. 
For a fitted model, let $\widehat{S}_{E\mid X}(u\mid x_i)$
denote the predicted event-time survival function for subject $i$. When a
point prediction is required, we use the predicted median survival time
\begin{align*}
    \widehat e_i \ = \ \inf\!\left\{u: \widehat{S}_{E\mid X}(u\mid x_i) \leq 1/2\right\},
\end{align*}
and define the corresponding risk score as
\begin{align*}
    \widehat{r}_i \ =\ -\widehat{e}_i,
\end{align*}
so that higher risk corresponds to shorter predicted survival.

\paragraph{Concordance Index.}
Discrimination measures whether a model correctly orders subjects by risk. We
use Harrell's concordance index~\citep{harrell1982evaluating}, computed over
comparable pairs in which subject $i$ experiences an observed event before
subject $j$:
\begin{align*}
    \operatorname{CI} \ = \
    \frac{
    \sum_{i=1}^{n}\sum_{j\neq i}
    \delta_i\,
    \mathbbm{1}[t_i<t_j]\,
    \mathbbm{1}\big[\widehat{r}_i>\widehat{r}_j\big]
    }{
    \sum_{i=1}^{n}\sum_{j\neq i}
    \delta_i\,
    \mathbbm{1}[t_i<t_j]
    }.
\end{align*}
Higher CI values indicate better performance.

\paragraph{Integrated Brier Score.}
The Brier score evaluates probabilistic accuracy at a target time $u$. 
Because the event status at $u$ may be unknown for subjects censored before $u$, we use inverse-probability-of-censoring weighting (IPCW)~\citep{graf1999assessment,robins2000correcting}.
Let $\widehat{S}_C$ denote the Kaplan-Meier estimate of the marginal censoring survival function, estimated from the training-side data. 
The IPCW Brier score is
\begin{align*}
    \operatorname{BS}(u) \ = \
    \frac{1}{n} \sum_{i=1}^{n}
    \left[
    \frac{
    \delta_i\mathbbm{1}[t_i\le u] \cdot \widehat{S}_{E\mid X}(u\mid x_i)^2
    }{
    \widehat{S}_C(t_i)
    }
    +
    \frac{
    \mathbbm{1}[t_i>u]\left(1- \widehat{S}_{E\mid X}(u\mid x_i)\right)^2
    }{
    \widehat{S}_C(u)
    }
    \right].
    \end{align*}
The integrated Brier score averages this quantity over an evaluation horizon $\tau$:
\begin{align}
\label{eq:ibs}
    \operatorname{IBS} \ = \
    \frac{1}{\tau} \int_{0}^{\tau} \operatorname{BS}(u)\,du.
\end{align}
In our experiments, $\tau$ is chosen from the training-side event-time support, so that the evaluation horizon is not determined by the held-out test set. 
Lower IBS indicates better performance.

\paragraph{Mean Absolute Error.}
To evaluate point prediction of event time, we use the mean absolute error based on pseudo-observations (MAE-PO)~\citep{qi2023effective}. 
For uncensored subjects, the observed event time $t_i$ can be used directly. 
For censored subjects, it constructs a pseudo-observation $\widetilde{e}_i$ that estimates the subject's contribution to the marginal Kaplan-Meier survival curve, and then weights the corresponding error by a confidence weight $w_i$. 
The resulting error takes the form
\begin{align*}
    \operatorname{MAE\text{-}PO} \ = \ 
    \frac{
    \sum_{i=1}^{n}
    w_i\,\left|\widehat{e}_i - \widetilde{e}_i\right|
    }{
    \sum_{i=1}^{n} w_i
    }.
\end{align*}
This produces an event-time error metric that can use both uncensored and
censored test subjects. Lower MAE value indicates better performance.

\paragraph{Distribution Calibration.}
Distribution calibration (D-calibration) evaluates whether the predicted survival
distribution is calibrated over the full time axis~\citep{haider2020effective}.
For an uncensored subject, define the probability integral transform value
\begin{align*}
    u_i \ = \ \widehat{S}_{E\mid X}(t_i \mid x_i). 
\end{align*}
If the predicted survival distributions are calibrated, then
$\{u_i:\delta_i=1\}$ should follow a standard uniform distribution on
$[0,1]$. In practice, we partition the probability range $[0,1]$ into $10$ equal-width bins. Uncensored subjects contribute to the bin containing
$\widehat{S}_{E\mid X}(t_i\mid x_i)$. For censored subjects, the event time is only known to
satisfy $e_i > t_i$, so their contribution is ``blurred'' over the portion of
the probability scale consistent with this information, namely
$\left[0, \widehat{S}_{E\mid X}(t_i\mid x_i)\right]$. 
D-calibration is then assessed by the chi-square statistic over the numbers in all the bins.
We report the chi-square statistic for each experiments, a smaller statistic indicates less evidence against distribution calibration.

\paragraph{Log-rank Reliability Test.}
We also use a log-rank goodness-of-fit test to assess whether predicted event
times are statistically aligned with the observed time-to-event data. The test
compares the observed test sample
\begin{align*}
    \mathcal A_{\mathrm{obs}} \ = \ \{(t_i, \delta_i)\}_{i=1}^{n}
\end{align*}
with the predicted event-time sample
\begin{align*}
    \mathcal{A}_{\mathrm{pred}} \ = \ \left\{\left(\widehat{e}_i, \widehat{\delta}_i=1 \right)\right\}_{i=1}^{n},
\end{align*}
where predicted median survival times are treated as uncensored event-time
predictions. 
We report the log-rank statistic for each experiments, a smaller statistic indicate closer agreement between the predicted event-time distribution and the observed test outcomes. 

All metrics are computed using the \texttt{SurvivalEVAL} package~\citep{qi2023survivaleval}.

\subsection{Compute Environment and Runtime Protocol}
\label{app:compute}

All benchmark experiments were executed on a Slurm-managed GPU cluster. Each
experiment was submitted as an independent job with a fixed resource budget: one NVIDIA L40S GPU (48 GB memory), one CPU core from Intel Xeon Gold 6448Y, 16 GB of system memory, and a maximum wall-clock time of 72 hours. The same resource allocation was used across models and datasets for fair evaluation.

The 72-hour wall-clock limit includes all computation required for a benchmark run, including model fitting, hyperparameter tuning when enabled, final refitting,
prediction on the held-out test set, and metric computation. 
If a run did not complete successfully within this budget, we treated the corresponding model-dataset run as failure. 
Failed runs were deemed as the worst (or equally worst if multiple models failed on this dataset) in the ranking procedure described in Appendix~\ref{app:benchmark_protocol}.

\subsection{Performance Reporting Protocol}
\label{app:benchmark_protocol}

We evaluate each model on repeated random train/validation/test splits. Each dataset is evaluated over $R=10$ repetitions. The experiment seed is set from 0 to 9 to these repetition, to support data split and model initialization (if needed). 

For each split $r$, the dataset is divided into a 70\% $\data^{tr + val}_{(r)}$ and a 30\% $\data^{te}_{(r)}$. 
We compute performance scores on the held-out test set via the metrics described in Appendix~\ref{app:metrics}. 
We report the mean and standard deviation across 10 split.

For dataset-level comparisons, each model is ranked separately within each
dataset and metric. 
Rank $1$ is assigned to the best value, with ties receiving the minimum tied rank. 
If a model does not produce a valid result for a given dataset-metric pair, it is assigned one rank below the worst completed method for that pair.

The cSVR baseline produces only a scalar event-time prediction rather than a full survival distribution. Consequently, distributional metrics such as IBS and D-calibration are not applicable to cSVR on any dataset.

Across the full benchmark, we evaluated $22$ models on $61$ datasets, yielding $21\times 61=1281$ model-dataset runs. 
Among these, $28$ runs failed to produce valid results, corresponding to a failure rate of $2.19\%$.
We summarize the failures below.

\begin{itemize}
    \item \textbf{CoxPH} failed on $13$ datasets: micro.censure, nki70, stagec,
    BMT, cancer, zinc, BCCardiotox, vlbw, PDM, actg, METABRIC, AIDS, and
    hdfail. In all cases, the failure was caused by a singular Hessian matrix of
    the Cox partial log-likelihood, which prevented the Newton-Raphson optimizer
    from computing a valid update.\footnote{This issue can often be mitigated by
    adding regularization. However, because we include CoxNet as the regularized
    Cox baseline, we intentionally keep CoxPH as the conventional unregularized
    Cox model.}

    \item \textbf{DeepSurv} failed on $1$ dataset, PDM. The fitted baseline hazard produced a degenerate survival curve equal to $1$ at all time points, resulting in identical predictions for all test instances.

    \item \textbf{CoxTime} failed on $1$ dataset, FRTCS, because the predicted survival matrix contained NaN values.

    \item \textbf{IWSG} failed on $5$ datasets: SUPPORT, MIMIC-IV\_all, hdfail, SEER\_brain, and SEER\_liver. In all cases, the model did not complete training within the 72-hour wall-clock limit.

    \item \textbf{DSM} failed on $3$ datasets: FRTCS, NPC, and MIMIC-IV\_all. For
    FRTCS and NPC, the predicted survival matrix contained NaN values. For
    MIMIC-IV\_all, the run did not complete within the $72$-hour wall-clock limit.


    \item \textbf{SurvivalMDN} failed on $3$ datasets: SEER\_liver, SEER\_brain,
    and hdfail. On hdfail, the run exceeded the available $48$ GB GPU memory. On
    SEER\_liver and SEER\_brain, the model did not complete within the $72$-hour
    wall-clock limit.

    \item \textbf{DeepAFT-Weibull} failed on $1$ dataset, FRTCS, because the model
    did not converge during training.

    \item \textbf{DeepAFT-Loglogistic} failed on $1$ dataset, FRTCS, because the
    model did not converge during training.
\end{itemize}

Let $\operatorname{rank}_{m,\data,q}$ denote the rank of model $m$ on dataset $\data$ for metric $q$.
For each model and metric, we summarize performance by the median rank across datasets,
\[
    \widetilde r_{m,q} \ = \
    \operatorname{median}_{\data}
    \operatorname{rank}_{m,\data,q}.
\]
To quantify uncertainty in this median rank, we use a nonparametric bootstrap over datasets. 
Specifically, for each bootstrap replicate $b=1,\ldots,B$, we sample datasets with replacement from
the benchmark collection, recompute the median rank,
\[
    \widetilde r^{(b)}_{m,q}
    \ = \
    \operatorname{median}_{\data \in \mathcal{B}^{(b)}}
    \operatorname{rank}_{m,\data,q},
\]
and report the 95\% bootstrap confidence interval as
\[
    \left[
    Q_{0.025}\!\left(\{\widetilde r^{(b)}_{m,q}\}_{b=1}^{B}\right),
    Q_{0.975}\!\left(\{\widetilde r^{(b)}_{m,q}\}_{b=1}^{B}\right)
    \right].
\]
The overall rank is computed by pooling ranks over all evaluation metrics before applying the same median-rank and bootstrap procedure.

\section{Benchmark Dataset Details}
\label{app:data}

We construct a large collection of real-world survival datasets from two sources.
First, we use datasets from the \texttt{SurvSet} package~\citep{drysdale2022survset}.
We exclude datasets that are longitudinal, contain competing risks rather than a single event of interest, or are high-dimensional in the sense that the number of features exceeds the number of samples. 
After applying these criteria, 57 \texttt{SurvSet} datasets remain. 
Second, we collect 24 additional survival datasets from textbooks, recent papers, and public software packages, restricting attention to datasets with at most 100,000 samples. 
This yields a total of 81 real-world datasets, which, to our knowledge, constitutes one of the largest survival-analysis benchmark collections considered in a single study. Summary statistics for all datasets are provided in Table~\ref{tab:dataset_summary}.

Among the 81 datasets, we designate 20 datasets as validation datasets for selecting the SurvivalPFN checkpoint during training.\footnote{SurvivalPFN is not trained or fine-tuned on these validation datasets, nor on any other real-world dataset; they are used only for checkpoint selection.}
The remaining 61 datasets are held out for the final benchmark and are not used for checkpoint selection. 
For each validation dataset, we use a deterministic split with 70\% of samples as the in-context training set and 30\% as the inference set.
At each training epoch, the current SurvivalPFN checkpoint is evaluated on all 20 validation datasets. 
We select the checkpoint with the best weighted average integrated Brier score, 
\begin{align*}
    \operatorname{Weighted-IBS}(\theta) \ 
    &=\
    \frac{
    \sum_{\data} \sqrt{N}\,
    \operatorname{IBS}_{\data}(\theta)
    }{
    \sum_{\data} \sqrt{N}
    }\\
    \theta^\ast \ 
    &= \ \arg\min_{\theta}
    \operatorname{Weighted-IBS}(\theta),
\end{align*}
where $\operatorname{IBS}_{\data}(\theta)$ is the IBS of checkpoint $\theta$ on the inference split of dataset $\data$, calculated using \Eqref{eq:ibs}. 
The square-root weighting gives larger datasets more influence while avoiding domination by the largest cohorts.

The validation datasets were chosen to cover a broad range of empirical regimes.
In particular, we considered sample size, censoring rate, and the tail survival
probability estimated by the Kaplan-Meier curve,
$\widehat{S}_{\mathrm{KM}}(t_{\max})$.

\begin{longtable}{@{}lrrrrrr@{}}
\caption{Summary of survival datasets used for checkpoint selection and held-out benchmarking. Features (cat. feat.) reports the total number of covariates, with the number of categorical covariates in parentheses.}
\label{tab:dataset_summary}\\
\toprule
\textbf{Dataset} &
\makecell{\textbf{Number of}\\\textbf{samples}} &
\makecell{\textbf{Features}\\\textbf{(cat. feat.)}} &
\makecell{\textbf{Missing}\\\textbf{features}} &
\makecell{\textbf{Missing}\\\textbf{rate}} &
\makecell{\textbf{Censoring}\\\textbf{rate}} &
$\widehat S(t_{\max})$ \\
\midrule
\endfirsthead
\toprule
\textbf{Dataset} &
\makecell{\textbf{Number of}\\\textbf{samples}} &
\makecell{\textbf{Features}\\\textbf{(cat.)}} &
\makecell{\textbf{Missing}\\\textbf{feat.}} &
\makecell{\textbf{Missing}\\\textbf{rate}} &
\makecell{\textbf{Censoring}\\\textbf{rate}} &
$\widehat S(t_{\max})$ \\
\midrule
\endhead
\midrule
\multicolumn{7}{r}{\emph{Continued on next page}}\\
\endfoot
\bottomrule
\endlastfoot
\multicolumn{7}{@{}l}{\textit{Validation datasets for checkpoint selection}}\\
\addlinespace[2pt]
ovarian & 26 & 4 (4) & 0 & 0.0\% & 53.8\% & 49.7\% \\
glioma & 37 & 4 (4) & 0 & 0.0\% & 37.8\% & 34.9\% \\
leukemia & 42 & 3 (2) & 0 & 0.0\% & 28.6\% & 18.9\% \\
pharmacoSmoking & 125 & 16 (16) & 0 & 0.0\% & 28.8\% & 28.8\% \\
d.oropha.rec & 192 & 24 (24) & 0 & 0.0\% & 27.6\% & 16.5\% \\
Pbc3 & 349 & 19 (19) & 4 & 0.4\% & 82.5\% & 63.4\% \\
retinopathy & 394 & 11 (11) & 0 & 0.0\% & 60.7\% & 53.1\% \\
Rossi & 432 & 7 (5) & 0 & 0.0\% & 73.6\% & 73.6\% \\
phpl04K8a & 442 & 21 (21) & 0 & 0.0\% & 46.6\% & 0.0\% \\
prostate & 502 & 25 (25) & 4 & 0.2\% & 29.5\% & 23.8\% \\
uis & 628 & 14 (14) & 3 & 0.6\% & 19.1\% & 16.6\% \\
grace & 1,000 & 5 (5) & 0 & 0.0\% & 67.6\% & 63.5\% \\
rdata & 1,040 & 6 (6) & 0 & 0.0\% & 47.4\% & 38.1\% \\
TRACE & 1,878 & 6 (6) & 0 & 0.0\% & 49.0\% & 43.9\% \\
Aids2 & 2,839 & 12 (12) & 0 & 0.0\% & 38.0\% & 5.8\% \\
UnempDur & 3,241 & 6 (6) & 0 & 0.0\% & 38.7\% & 10.5\% \\
smarto & 3,873 & 34 (34) & 16 & 4.8\% & 88.1\% & 72.5\% \\
dataDIVAT1 & 5,943 & 16 (16) & 0 & 0.0\% & 83.6\% & 0.0\% \\
oldmort & 6,495 & 14 (14) & 0 & 0.0\% & 69.7\% & 0.0\% \\
prostateSurvival & 14,294 & 6 (6) & 0 & 0.0\% & 94.4\% & 81.4\% \\
\midrule
\multicolumn{7}{@{}l}{\textit{Held-out test datasets for benchmarking}}\\
\addlinespace[2pt]
Bergamaschi & 82 & 10 (10) & 0 & 0.0\% & 65.9\% & 58.3\% \\
larynx & 90 & 4 (3) & 0 & 0.0\% & 44.4\% & 29.7\% \\
lupus & 95 & 5 (2) & 1 & 0.2\% & 70.5\% & 36.8\% \\
micro.censure & 117 & 81 (81) & 0 & 0.0\% & 77.8\% & 42.4\% \\
cgd & 128 & 23 (23) & 0 & 0.0\% & 65.6\% & 47.6\% \\
veteran & 137 & 8 (8) & 0 & 0.0\% & 6.6\% & 0.0\% \\
nki70 & 144 & 76 (76) & 0 & 0.0\% & 66.7\% & 48.0\% \\
stagec & 146 & 18 (18) & 2 & 0.4\% & 63.0\% & 55.8\% \\
burn & 154 & 13 (13) & 0 & 0.0\% & 68.8\% & 48.0\% \\
BMT & 187 & 39 (24) & 11 & 1.1\% & 54.5\% & 53.3\% \\
WPBC & 198 & 32 (1) & 1 & 0.1\% & 76.3\% & 65.3\% \\
Melanoma & 205 & 5 (5) & 0 & 0.0\% & 72.2\% & 64.5\% \\
hepatoCellular & 227 & 43 (43) & 26 & 33.6\% & 57.3\% & 51.2\% \\
cancer & 228 & 28 (28) & 4 & 1.0\% & 27.6\% & 5.0\% \\
NCCTG & 228 & 8 (1) & 6 & 3.7\% & 27.6\% & 5.0\% \\
mgus & 241 & 12 (12) & 4 & 8.7\% & 6.6\% & 1.8\% \\
e1684 & 284 & 3 (3) & 0 & 0.0\% & 31.0\% & 28.3\% \\
HFCR & 299 & 11 (5) & 0 & 0.0\% & 67.9\% & 57.6\% \\
ova & 358 & 11 (11) & 0 & 0.0\% & 25.7\% & 22.2\% \\
diabetes & 394 & 4 (4) & 0 & 0.0\% & 60.7\% & 53.0\% \\
PBC & 418 & 17 (7) & 12 & 14.5\% & 61.5\% & 35.3\% \\
zinc & 431 & 20 (20) & 0 & 0.0\% & 81.2\% & 79.0\% \\
Unemployment & 452 & 7 (7) & 0 & 0.0\% & 43.4\% & 5.8\% \\
whas500 & 461 & 17 (17) & 0 & 0.0\% & 61.8\% & 0.0\% \\
cost & 518 & 13 (13) & 0 & 0.0\% & 22.0\% & 22.0\% \\
BCCardiotox & 531 & 25 (15) & 22 & 5.7\% & 89.8\% & 69.0\% \\
GBM & 591 & 10 (7) & 3 & 4.5\% & 17.1\% & 0.0\% \\
vlbw & 617 & 41 (41) & 5 & 1.7\% & 82.7\% & 0.0\% \\
GBSG2 & 686 & 8 (2) & 0 & 0.0\% & 56.4\% & 34.3\% \\
FRTCS & 697 & 13 (13) & 2 & 0.1\% & 89.7\% & 57.2\% \\
kidney\_transplant & 863 & 4 (3) & 0 & 0.0\% & 83.8\% & 72.4\% \\
dataOvarian1 & 912 & 162 (162) & 0 & 0.0\% & 40.4\% & 11.9\% \\
colon & 929 & 12 (12) & 1 & 0.2\% & 51.3\% & 45.5\% \\
credit & 1,000 & 31 (20) & 1 & 0.6\% & 30.0\% & 5.9\% \\
PDM & 1,000 & 8 (5) & 0 & 0.0\% & 60.3\% & 0.0\% \\
LeukSurv & 1,043 & 29 (29) & 0 & 0.0\% & 15.7\% & 4.4\% \\
actg & 1,151 & 17 (17) & 0 & 0.0\% & 91.7\% & 90.1\% \\
COVID & 1,422 & 9 (2) & 0 & 0.0\% & 90.5\% & 1.1\% \\
WHAS & 1,638 & 6 (4) & 0 & 0.0\% & 57.9\% & 35.7\% \\
dataDIVAT2 & 1,837 & 4 (4) & 0 & 0.0\% & 68.3\% & 31.1\% \\
scania & 1,931 & 8 (8) & 0 & 0.0\% & 43.8\% & 15.9\% \\
churn & 1,958 & 19 (10) & 0 & 0.0\% & 52.4\% & 24.4\% \\
METABRIC & 1,981 & 79 (73) & 0 & 0.0\% & 55.2\% & 11.6\% \\
AIDS & 2,139 & 22 (14) & 0 & 0.0\% & 24.4\% & 0.0\% \\
NACD & 2,396 & 48 (33) & 0 & 0.0\% & 36.4\% & 12.5\% \\
rott2 & 2,982 & 12 (12) & 0 & 0.0\% & 57.3\% & 26.5\% \\
divorce & 3,371 & 4 (4) & 0 & 0.0\% & 69.4\% & 55.7\% \\
acath & 3,504 & 3 (3) & 1 & 11.9\% & 33.4\% & 0.0\% \\
NWTCO & 4,028 & 6 (5) & 0 & 0.0\% & 85.8\% & 84.9\% \\
dataDIVAT3 & 4,267 & 16 (16) & 0 & 0.0\% & 94.4\% & 84.7\% \\
Framingham & 4,699 & 17 (17) & 2 & 0.1\% & 68.7\% & 60.5\% \\
NPC & 6,449 & 9 (3) & 0 & 0.0\% & 80.8\% & 75.6\% \\
Dialysis & 6,805 & 72 (72) & 0 & 0.0\% & 76.4\% & 58.7\% \\
FLCHAIN & 7,871 & 23 (2) & 1 & 0.7\% & 72.5\% & 68.2\% \\
MSKCC & 8,130 & 206 (199) & 2 & 0.0\% & 70.3\% & 34.8\% \\
SUPPORT & 9,105 & 31 (11) & 10 & 4.0\% & 31.9\% & 24.1\% \\
employee & 11,991 & 16 (9) & 0 & 0.0\% & 83.4\% & 50.8\% \\
MIMIC-IV\_all & 38,520 & 91 (6) & 0 & 0.0\% & 66.7\% & 0.0\% \\
hdfail & 52,422 & 88 (88) & 0 & 0.0\% & 94.5\% & 56.1\% \\
SEER\_brain & 73,703 & 10 (4) & 0 & 0.0\% & 40.1\% & 26.6\% \\
SEER\_liver & 82,841 & 14 (1) & 0 & 0.0\% & 37.6\% & 18.0\% \\
\end{longtable}

We briefly describe below the 24 datasets that are not taken from
\texttt{SurvSet}. For the \texttt{SurvSet} datasets, please refer to~\citet{drysdale2022survset}.
\begin{itemize}
    \item \textit{leukemia}: The \textit{leukemia} dataset records remission survival for 42 patients with sex, treatment assignment, and log white blood cell count~\citep{kleinbaum1996survival}. 

    \item \textit{larynx}: The \textit{larynx} cancer dataset follows 90 male patients from first treatment until death or study end, with disease stage, age, and diagnosis year~\citep{kardaun1983statistical}. 

    \item \textit{lupus}: The \textit{lupus} dataset records survival times and diagnostic features for systemic lupus erythematosus patients~\citep{merrell1955determination}. 

    \item \textit{BMT}: The Bone Marrow Transplant (\textit{BMT}) Children dataset describes pediatric patients with malignant and nonmalignant hematologic diseases who underwent unrelated-donor allogeneic hematopoietic stem cell transplantation~\citep{sikora2019guider}.

    \item \textit{NCCTG}: The \textit{NCCTG} lung cancer dataset follows advanced lung cancer patients with physician-rated and patient-rated performance scores~\citep{loprinzi1994prospective}.
    
    \item \textit{HFCR}: The Heart Failure Clinical Records (\textit{HFCR}) dataset contains follow-up outcomes and clinical measurements for 299 heart failure patients~\citep{chicco2020machine}.

    \item \textit{Rossi}: The \textit{Rossi} dataset follows 432 Maryland prison releasees for one year to study recidivism after financial-aid treatment assignment~\citep{rossi1980money}.

    \item \textit{BCCardiotox}: \textit{BCCardiotox} follows HER2+ breast cancer patients treated with potentially cardiotoxic therapies and records time to cancer therapy-related cardiac dysfunction~\citep{pineiro2023cardiotoxicity}.

    \item \textit{GBM}: The TCGA glioblastoma multiforme (\textit{GBM}) dataset contains clinical survival information for primary brain tumor patients~\citep{cancer2008comprehensive}. 

    \item \textit{kidney\_transplant}: The \textit{kidney\_transplant} dataset records post-transplant death times with recipient age, gender, and race~\citep{klein2003survival}.

    \item \textit{credit}: The \texttt{PySurvival}~\citep{pysurvival_cite} credit-risk dataset adapts the German Credit data to model the time until a loan is fully repaid.

    \item \textit{PDM}: The \texttt{PySurvival}~\citep{pysurvival_cite} predictive-maintenance (\textit{PDM}) dataset models the time until an industrial machine breaks.

    \item \textit{COVID}: The COVID-19 Asian discharge dataset models time to hospital discharge for patients with COVID-19 diagnosis~\citep{kumar2022learning}. 

    \item \textit{WHAS}: The Worcester Heart Attack Study (\textit{WHAS}) follows acute myocardial infarction patients after hospital admission~\citep{hosmer2008applied}.

    \item \textit{churn}: The \texttt{PySurvival}~\citep{pysurvival_cite} churn dataset models when SaaS customers stop their monthly subscription. 

    \item \textit{METABRIC}: \textit{METABRIC} profiles primary breast tumors with long-term clinical follow-up for breast-cancer prognosis~\citep{curtis2012genomic}.
    
    \item \textit{AIDS}: ACTG 175 records HIV-infected adults randomized to nucleoside monotherapy or combination therapy~\citep{hammer1996trial}.

    \item \textit{NACD}: The Northern Alberta Cancer Dataset contains clinical records for several cancer sites, which used to predict the death from cancer onset~\citep{yu2011learning}. 

    \item \textit{NPC}: The nasopharyngeal carcinoma dataset follows patients from Sun Yat-sen University Cancer Center for progression-free survival after radiotherapy with or without chemotherapy~\citep{tang2016establishment}. Original training and validation cohorts are combined.

    \item \textit{MSKCC}: The MSK-IMPACT cohort contains targeted sequencing and clinical annotations from more than 10,000 advanced cancer cases~\citep{zehir2017mutational}. 

    \item \textit{MIMIC-IV\_all}: \textit{MIMIC-IV\_all} is derived from the MIMIC-IV critical care database, which contains de-identified electronic health records for patients admitted to intensive care units~\citep{johnson2023mimic}. We use the all-cause mortality cohort curated by~\citet{qi2023effective}, restricting to patients who survived at least 24 hours after ICU admission.

    \item \textit{employee}: The \texttt{PySurvival}~\citep{pysurvival_cite} employee-retention dataset models when employees leave a company using HR and workload attributes.

    \item \textit{SEER\_brain}: \textit{SEER\_brain} is a brain cancer cohort derived from the Surveillance, Epidemiology, and End Results (SEER) Program registry. The task is to predict time from cancer diagnosis to death or censoring, using the cohort curated by~\citet{qi2024conformalized}.
    
    \item \textit{SEER\_liver}: \textit{SEER\_liver} is the corresponding liver cancer cohort from SEER, with the same time-to-death prediction target and curation protocol~\citep{qi2024conformalized}.
    
\end{itemize}

\section{Additional Experimental Results}
\subsection{RQ1: Predictive Performance}
\label{app:rq1}

Figure~\ref{fig:main_results} has demonstrate the overall performance across 61 benchmark datasets. 
The intervals reflect uncertainty in the estimated median rank across the benchmark collection, rather than the variability of ranks themselves.
In this appendix, we further analyze the benchmark results by stratifying datasets according to sample size and censoring rate. 

The sample-size stratification reveals that SurvivalPFN is particularly effective in the small-data regime. On small datasets (Figure~\ref{fig:best_model_small_data}), SurvivalPFN is clearly separated from most baselines in overall rank and performs strongly across all five metrics, including IBS, CI, D-calibration, MAE, and Log-Rank.
Traditional statistical survival methods and tree-based methods performance in the second tier, while neural-network-based methods generally performance the worst. 

\begin{figure}[t]
    \centering
    \includegraphics[width=\linewidth]{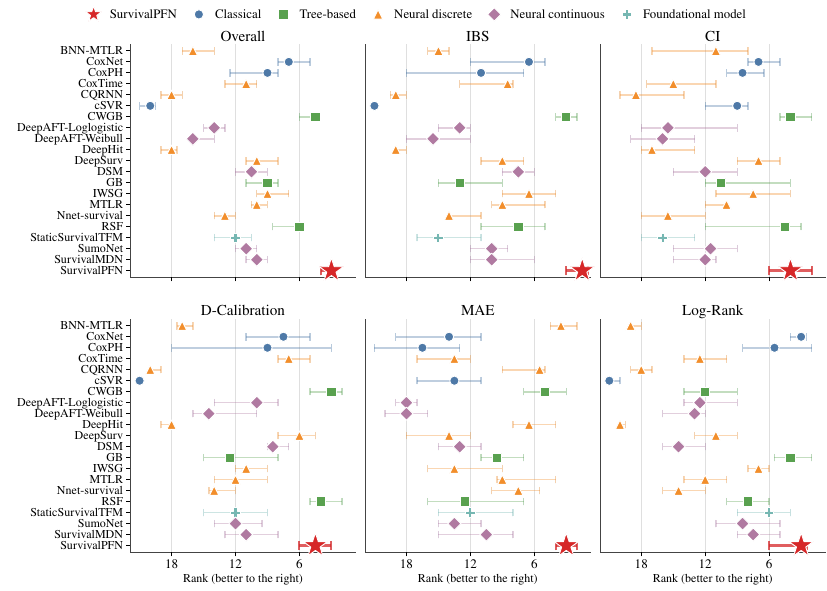}
    \caption{\textbf{Model ranks across 24 small-size datasets ($N < 500$).} Points/stars denote median ranks across datasets, with horizontal bars showing 95\% bootstrap confidence intervals for the median rank.}
    \label{fig:best_model_small_data}
\end{figure}

For medium-sized datasets (Figure~\ref{fig:best_model_mid_data}), SurvivalPFN remains among the strongest methods overall, with especially competitive performance on IBS, Log-Rank. 
However, the gap to the best baselines becomes smaller, and several nerual-network-based methods (\eg, MTLR, SurvivalMDN) become competitive on the overall performance and also on individual metrics such as CI or D-calibration. 

On large datasets (Figure~\ref{fig:best_model_large_data}), the advantage of SurvivalPFN decreases further: its overall, IBS, and CI ranks move leftward compared with the small-data setting, although it remains competitive on MAE and Log-Rank. 
This trend suggests a size-regime trade-off. 
As more observations become available, deep-learning-based estimators can benefit more directly from the larger sample size, whereas SurvivalPFN is constrained by finite context length and the need to summarize large tables during inference.

\begin{figure}[t]
    \centering
    \includegraphics[width=\linewidth]{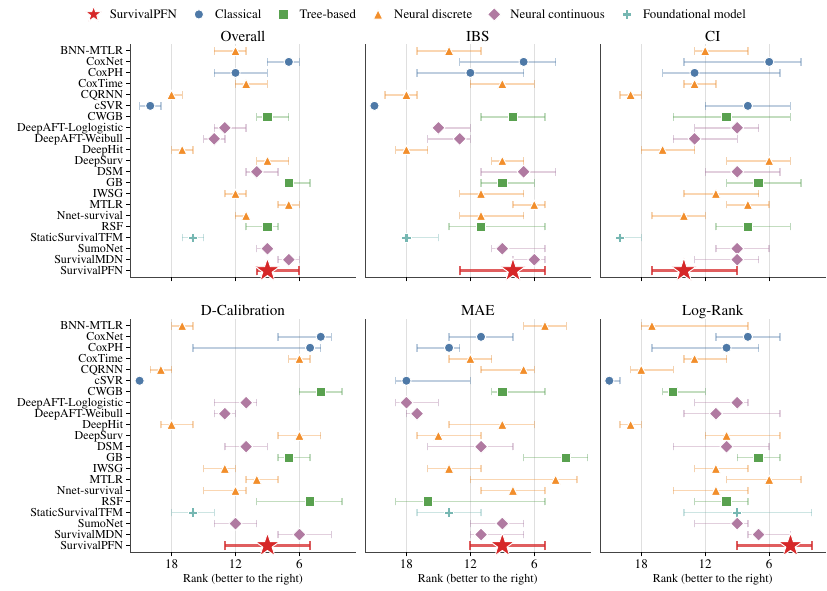}
    \caption{\textbf{Model ranks across 27 medium-size datasets ($500 \leq N < 5000$).} Points/stars denote median ranks across datasets, with horizontal bars showing 95\% bootstrap confidence intervals for the median rank.}
    \label{fig:best_model_mid_data}
\end{figure}

\begin{figure}[t]
    \centering
    \includegraphics[width=\linewidth]{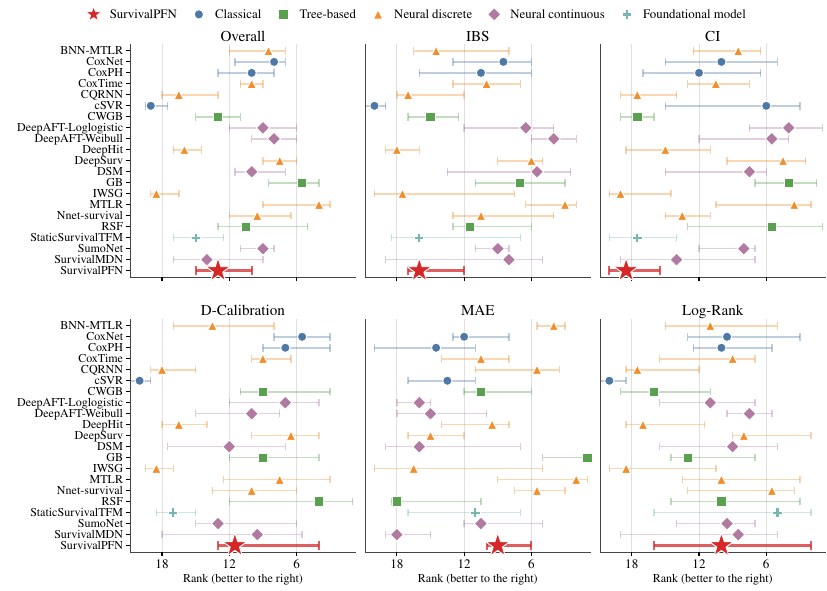}
    \caption{\textbf{Model ranks across 10 medium-size datasets ($N \geq 5000$).} Points/stars denote median ranks across datasets, with horizontal bars showing 95\% bootstrap confidence intervals for the median rank.}
    \label{fig:best_model_large_data}
\end{figure}

In contrast, the censoring-rate stratification shows substantially more stable behavior. Across low-, medium-, and high-censoring subsets (Figures~\ref{fig:best_model_low_censor}-\ref{fig:best_model_high_censor}), SurvivalPFN remains in the top group overall and retains strong performance on IBS, MAE, and Log-Rank. This consistency is encouraging because censoring affects the available event-time information and can differentially impact discrimination, calibration, and distributional accuracy metrics. The results suggest that the synthetic prior used to train SurvivalPFN provides useful robustness across censoring regimes. Although the strongest baseline varies across metrics and censoring levels, no single baseline matches SurvivalPFN's overall consistency across dataset strata, censoring strata, and evaluation metrics.

\begin{figure}[t]
    \centering
    \includegraphics[width=\linewidth]{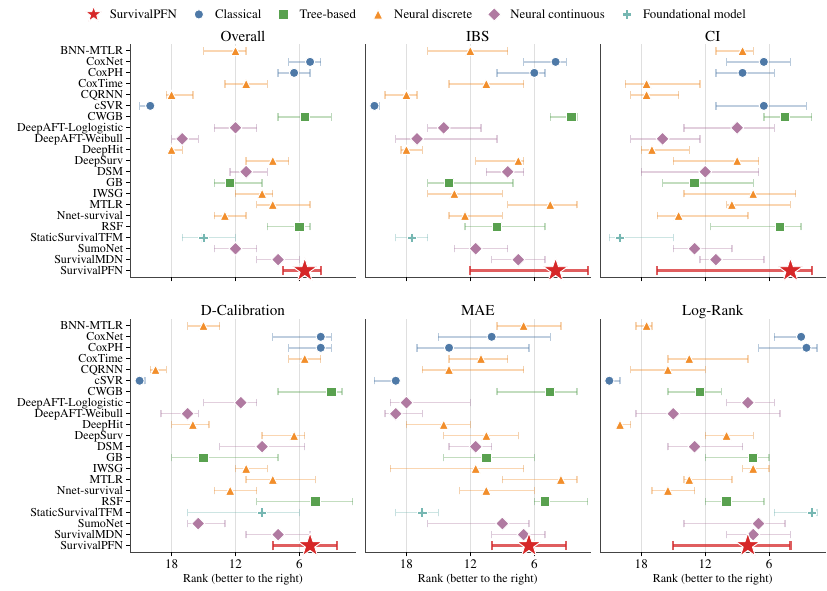}
    \caption{\textbf{Model ranks across 12 low-censoring-rate datasets (censoring rate $< 33\%$).} Points/stars denote median ranks across datasets, with horizontal bars showing 95\% bootstrap confidence intervals for the median rank.}
    \label{fig:best_model_low_censor}
\end{figure}

\begin{figure}[t]
    \centering
    \includegraphics[width=\linewidth]{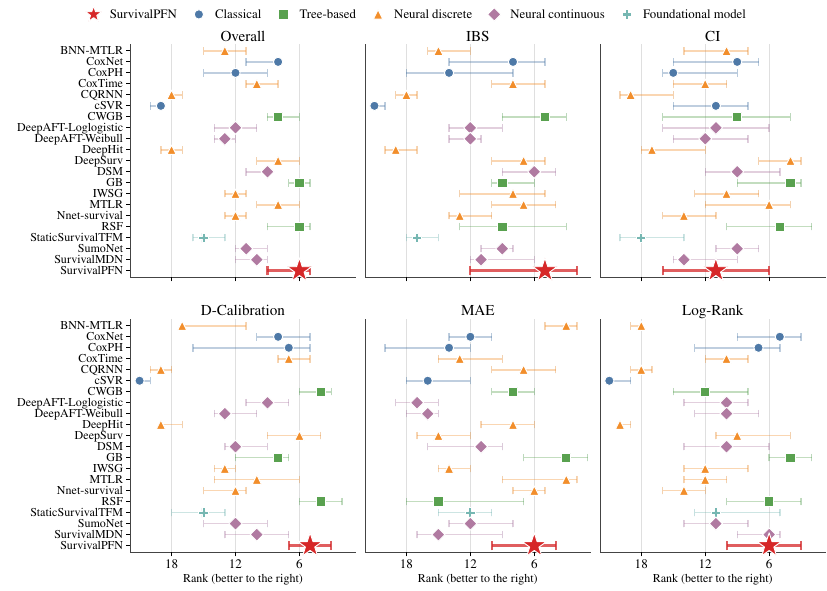}
    \caption{\textbf{Model ranks across 25 medium-censoring-rate datasets (censoring rate $\geq 33\%$ and $< 67\%$).} Points/stars denote median ranks across datasets, with horizontal bars showing 95\% bootstrap confidence intervals for the median rank.}
    \label{fig:best_model_mid_censor}
\end{figure}

\begin{figure}[t]
    \centering
    \includegraphics[width=\linewidth]{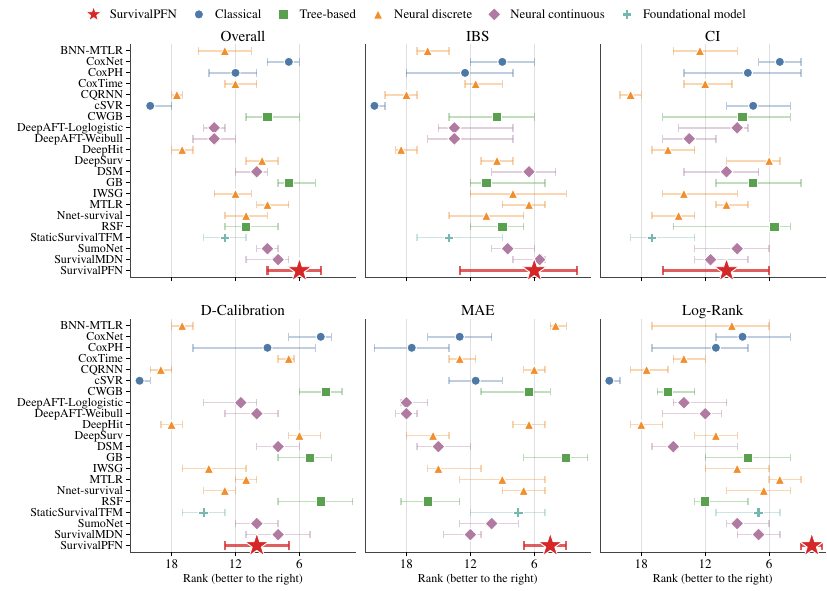}
    \caption{\textbf{Model ranks across 24 high-censoring-rate datasets (censoring rate $\geq 67\%$).} Points/stars denote median ranks across datasets, with horizontal bars showing 95\% bootstrap confidence intervals for the median rank. }
    \label{fig:best_model_high_censor}
\end{figure}

\subsection{RQ2: Computational Efficiency}
\label{app:rq2}
For the runtime comparison in Figure~\ref{fig:teaser}, datasets with at least one failed run are excluded from this runtime aggregation, so every method is compared on the same set of completed datasets. 
This prevents methods from appearing artificially faster because they crashed, timed out, or failed to produce valid predictions on more demanding datasets. Predictive-performance rankings still follow the failure handling protocol in Appendix~\ref{app:benchmark_protocol}.

\subsection{RQ3: Sensitivity to Training-Set Size}
\label{app:rq3}

Figure~\ref{fig:diff_ratio_appendix} extends the training-ratio analysis to 16 additional datasets. 
Across these datasets, increasing the training ratio improves all methods, with IBS decreasing and CI increasing most sharply when moving from very small context sizes to moderate context sizes. 
SurvivalPFN remains stable across ratios and is frequently among the best methods on both metrics, especially in low-data regimes. 

\begin{figure}[p]
    \centering
    \captionsetup[subfigure]{font=footnotesize}

    \begin{subfigure}{0.48\linewidth}
        \centering
        \includegraphics[width=\linewidth]{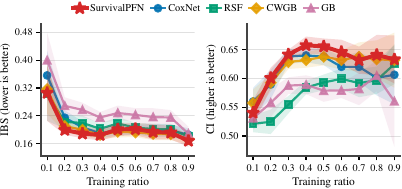}
        \vspace{-1.6em}
        \caption{larynx}
        \label{fig:panel-01}
    \end{subfigure}
    \hfill
    \begin{subfigure}{0.48\linewidth}
        \centering
        \includegraphics[width=\linewidth]{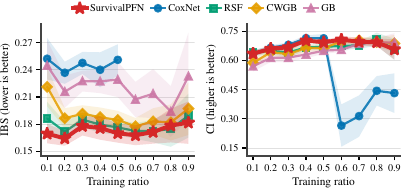}
        \vspace{-1.6em}
        \caption{nki70}
        \label{fig:panel-02}
    \end{subfigure}

    \vspace{0.4em}

    \begin{subfigure}{0.48\linewidth}
        \centering
        \includegraphics[width=\linewidth]{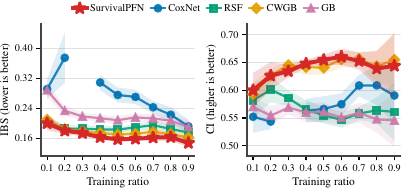}
        \vspace{-1.6em}
        \caption{WPBC}
        \label{fig:panel-05}
    \end{subfigure}
    \hfill
    \begin{subfigure}{0.48\linewidth}
        \centering
        \includegraphics[width=\linewidth]{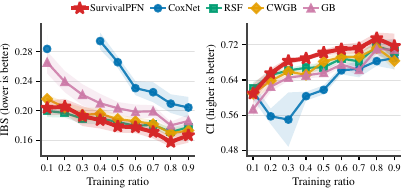}
        \vspace{-1.6em}
        \caption{hepatoCellular}
        \label{fig:panel-06}
    \end{subfigure}

    \vspace{0.4em}

    \begin{subfigure}{0.48\linewidth}
        \centering
        \includegraphics[width=\linewidth]{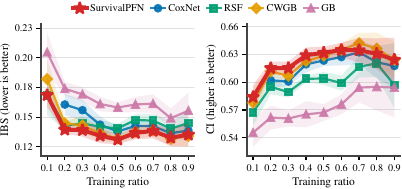}
        \vspace{-1.6em}
        \caption{NCCTG}
        \label{fig:panel-09}
    \end{subfigure}
    \hfill
    \begin{subfigure}{0.48\linewidth}
        \centering
        \includegraphics[width=\linewidth]{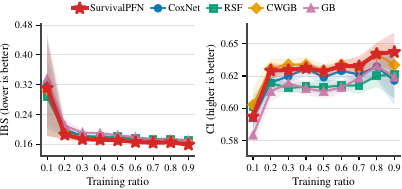}
        \vspace{-1.6em}
        \caption{ova}
        \label{fig:panel-10}
    \end{subfigure}

    \vspace{0.4em}

    \begin{subfigure}{0.48\linewidth}
        \centering
        \includegraphics[width=\linewidth]{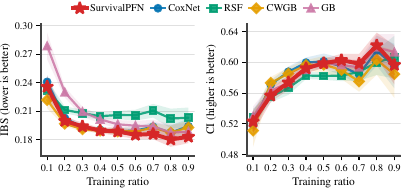}
        \vspace{-1.6em}
        \caption{diabetes}
        \label{fig:panel-11}
    \end{subfigure}
    \hfill
    \begin{subfigure}{0.48\linewidth}
        \centering
        \includegraphics[width=\linewidth]{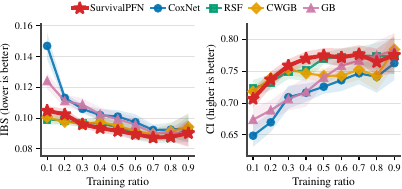}
        \vspace{-1.6em}
        \caption{zinc}
        \label{fig:panel-12}
    \end{subfigure}

    \vspace{0.4em}

    \begin{subfigure}{0.48\linewidth}
        \centering
        \includegraphics[width=\linewidth]{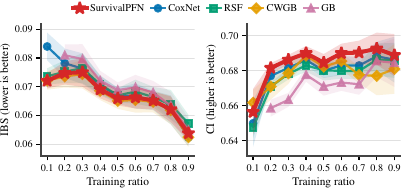}
        \vspace{-1.6em}
        \caption{GBM}
        \label{fig:panel-13}
    \end{subfigure}
    \hfill
    \begin{subfigure}{0.48\linewidth}
        \centering
        \includegraphics[width=\linewidth]{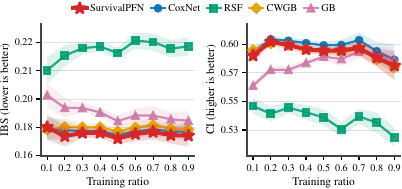}
        \vspace{-1.6em}
        \caption{dataDIVAT2}
        \label{fig:panel-14}
    \end{subfigure}

    \vspace{0.4em}

    \begin{subfigure}{0.48\linewidth}
        \centering
        \includegraphics[width=\linewidth]{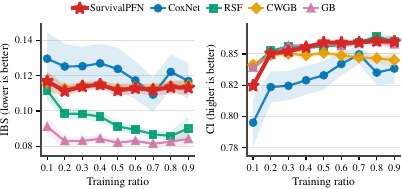}
        \vspace{-1.6em}
        \caption{churn}
        \label{fig:panel-15}
    \end{subfigure}
    \hfill
    \begin{subfigure}{0.48\linewidth}
        \centering
        \includegraphics[width=\linewidth]{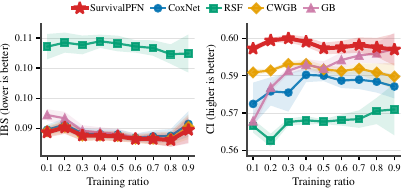}
        \vspace{-1.6em}
        \caption{acath}
        \label{fig:panel-16}
    \end{subfigure}

    \caption{Sensitivity to the training/context ratio across selected 16 datasets. }
    \label{fig:diff_ratio_appendix}
\end{figure}

\subsection{RQ4: Compare with General Tabular Foundational Models}
\label{app:rq4}
This section compares the SurvivalPFN, StaticSurvialTFM with other general TFMs. 
Specifically, we includes the most advanced TFMs including:
TabPFN v2.5~\citep{grinsztajn2025tabpfn}, TabICL v2~\citep{qu2026tabiclv2}, MITRA~\citep{zhang2025mitra}, and TabDPT~\citep{ma2024tabdpt} regressors.

For the these TFM regression baselines, we train only on uncensored training examples because these models cannot natively deal with right-censoring datasets. 
TabPFN and TabICL are used as quantile regressors: they predict event-time quantiles, which are monotonized and converted into survival curves for distributional evaluation. 
MITRA and TabDPT are used as point regressors: they predict a single event time per test subject (just like cSVR). 
Since point predictions do not define a full survival distribution, distributional metrics such as IBS and D-calibration are not calculated and ranked as the worst among all the methods, while CI, MAE, and log-rank style comparisons are computed from the predicted event times.

For StaticSurvialTFM~\citep{kim2026tabular}, it is a static fomula that can convert any classifier to survival predictor. 
We instantiate this static formulation with TabDPT and MITRA classifier backbones, predict failure probabilities over the cutoff grid, convert them to survival probabilities, and enforce monotone survival curves.
We choose TabDPT to match with the model architecture of SurvivalPFN (for a fair comparison). 
We include MITRA as it is the best performing backbone described in~\citet{kim2026tabular}.

The results present here uses the same evaluation protocal as described in Appendix~\ref{app:rq1}. 
Figure~\ref{fig:tfm_appendix} expands the comparison in Figure~\ref{fig:tfm_main} by including both general TFMs and the StaticSurvivalTFM wrapper instantiated with TabDPT and MITRA. Overall, SurvivalPFN remains the strongest and most consistent method: it achieves the best aggregate rank and ranks first or near-first across nearly all metrics. The largest gains appear for IBS and D-calibration, where SurvivalPFN clearly outperforms both direct TFM baselines and StaticSurvivalTFM variants, indicating better probabilistic survival estimation and calibration. SurvivalPFN also performs best on CI and log-rank, showing that its advantage extends beyond distributional accuracy to risk ranking and group separation.

The StaticSurvivalTFM performance is really sensitive to the backbone TFM model -- which aligns the findings in~\citep{kim2026tabular}.
Using MITRA as the backbone improve over using TabDPT, especially for CI and log-rank, confirming that survival-specific label construction is helpful. 
However, their performance is less stable across metrics: StaticSurvivalTFM (MITRA) is competitive on MAE, CI, but does not match SurvivalPFN on IBS, D-calibration and Log-rank; StaticSurvivalTFM (TabDPT) performs well on log-rank but is weaker on other. 
In contrast, the direct regression baselines -- TabPFN, TabICL, MITRA, and TabDPT -- are consistently worse, despite being strong general tabular predictors. 

These results support the main conclusion of \textbf{RQ4}: survival prediction benefits from a foundation model trained with survival-specific supervision and censoring-aware synthetic tasks, rather than relying only on generic tabular in-context learning.

\begin{figure}
    \centering
    \includegraphics[width=\linewidth]{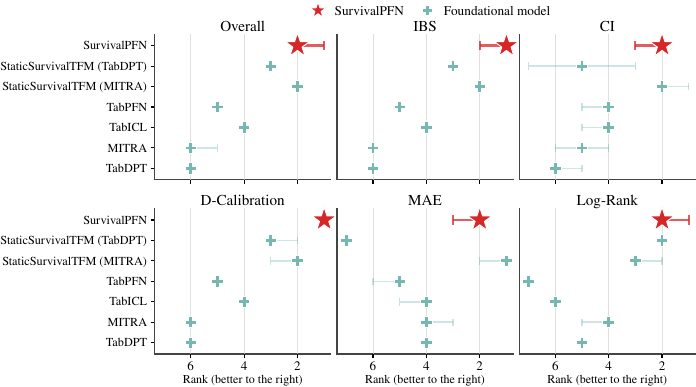}
    \caption{\textbf{Compare SurvivalPFN with general TFMs across 61 benchmark datasets.} Plotting conventions follow Figure~\ref{fig:main_results}.}
    \label{fig:tfm_appendix}
\end{figure}

\subsection{RQ5: Ablation Studies}
\label{app:rq5}

\begin{table*}[t]
\centering
\caption{\textbf{SurvivalPFN ablation configurations.} Each row corresponds to one pretrained SurvivalPFN
variant. The internal checkpoint-path column from the experiment log is omitted.
``Surv.-dist.'' denotes the survival-distribution prior. NLL denotes the one-hot
discrete negative log-likelihood over transformed-time bins; CE denotes the smoothed histogram cross-entropy loss.}
\label{tab:survivalpfn_ablation_configs}
\begin{tabular}{@{}ccllllc@{}}
\toprule
\textbf{Model}
& \makecell{\textbf{Predictive}\\\textbf{Pretrain}}
& \makecell{\textbf{Prior}}
& \makecell{\textbf{Time}\\\textbf{Transformation}}
& \makecell{\textbf{Loss}}
& \makecell{\textbf{Query}\\\textbf{Schedule}}
& \makecell{\textbf{Variable}\\\textbf{Train Ratio}} \\
\midrule
v01$^{\dagger}$ & \benchcmark & Surv.-dist.   & \texttt{lognormal2normal} & NLL & Both       & \benchxmark \\
v02              & \benchcmark & Surv.-dist.   & \texttt{time2quantile}    & NLL & Both       & \benchxmark \\
v03              & \benchxmark & Surv.-dist.   & \texttt{time2quantile}    & CE  & Random     & \benchcmark \\
v04              & \benchcmark & Surv.-dist.   & \texttt{time2quantile}    & CE  & Random     & \benchcmark \\
v05              & \benchcmark & Surv.-dist.   & \texttt{time2quantile}    & NLL & Random     & \benchcmark \\
v06              & \benchcmark & Surv.-dist.   & \texttt{lognormal2normal} & NLL & Random     & \benchcmark \\
v07              & \benchcmark & Kitchen-sink  & \texttt{lognormal2normal} & CE  & Random     & \benchcmark \\
v08              & \benchcmark & Kitchen-sink  & \texttt{lognormal2normal} & NLL & Event-only & \benchcmark \\
v09              & \benchcmark & Surv.-dist.   & \texttt{lognormal2normal} & CE  & Random     & \benchcmark \\
v10              & \benchcmark & Surv.-dist.   & \texttt{lognormal2normal} & NLL & Event-only & \benchcmark \\
v11              & \benchcmark & Naive         & \texttt{lognormal2normal} & CE  & Random     & \benchcmark \\
v12              & \benchcmark & Naive         & \texttt{lognormal2normal} & NLL & Event-only & \benchcmark \\
v13              & \benchxmark & Naive         & \texttt{lognormal2normal} & CE  & Random     & \benchcmark \\
v14              & \benchxmark & Naive         & \texttt{lognormal2normal} & NLL & Event-only & \benchcmark \\
\bottomrule
\end{tabular}

\vspace{0.5em}
\begin{minipage}{0.97\linewidth}
\footnotesize
$^{\dagger}$Best validation run; this checkpoint is used as the default SurvivalPFN model elsewhere
in the paper.
\end{minipage}
\end{table*}
\paragraph{SurvivalPFN Ablation Configurations.}
Table~\ref{tab:survivalpfn_ablation_configs} summarizes the SurvivalPFN variants used in the ablation study. 
Each row corresponds to one pretrained checkpoint and is defined by six configuration choices.

\textbf{Predictive Pretraining} specifies whether the model is initialized from the predictive PFN-style pretraining phase before survival-specific training. 
\benchcmark{} means that the model first undergoes the general predictive pretraining stage described in Appendix~\ref{app:model_details}, and is then further trained with the survival phase. 
A value of \benchxmark{} means that survival-phase training starts without this predictive initialization. 
This ablation tests whether generic PFN-style in-context predictive pretraining improves downstream survival prediction.

\textbf{Prior} specifies the synthetic survival prior used to generate the right-censored pretraining tasks, as described in Appendix~\ref{app:prior_generation}. 
We consider four possible prior families -- the \emph{naive prior}, the \emph{survival-distribution prior}, the \emph{mixture prior} and the \emph{kitchen-sink prior}.

\textbf{Time transformation} specifies the monotone transformation applied to event and censoring times before discretization, as described in Appendix~\ref{app:monotone_transformation}. 
We try the \texttt{lognormal2normal} and the \texttt{time2quantile} transformations.

\textbf{Loss} specifies how the discretized predictive distribution is trained in transformed-time space, as described in Appendix~\ref{app:train_objective}. 
The \emph{NLL} setting uses a one-hot discrete negative log-likelihood, assigning all target mass to the bin containing the latent target time.
The \emph{CE} setting uses the smoothed histogram cross-entropy loss, where the latent target time is converted into a narrow Gaussian-smoothed histogram over bins.

\textbf{Query schedule} specifies how the query indicator $\widetilde{\delta}^{\ast}$ is selected during training, following Section~\ref{sec:method}. 
We try the \emph{event-only}, \emph{both}, and \emph{random} strategies.

\textbf{Variable train ratio} specifies whether the ratio between context/training samples and query/inference samples is varied during synthetic pretraining. 
A value of \benchcmark{} means that this ratio is randomized across synthetic tasks, exposing the model to different amounts of context
information and encouraging robustness to varying downstream train/test splits. 
A value of \benchxmark{} means that the ratio is fixed at 70\%/30\% during training.

Together, these choices define a full configuration space of
$2 \times 4 \times 2 \times 2 \times 3 \times 2 = 192$ possible variants (corresponding respectively to predictive pretraining, prior family, time transformation, loss, query schedule, and variable train ratio). 
Exhaustively training all variants would be computationally expensive, so we selectively evaluate the representative configurations in Table~\ref{tab:survivalpfn_ablation_configs}. 
The model marked with $\dagger$ is the best validation run and is used as the default SurvivalPFN checkpoint elsewhere in the paper.

\begin{figure}[t]
    \centering
    \includegraphics[width=\linewidth]{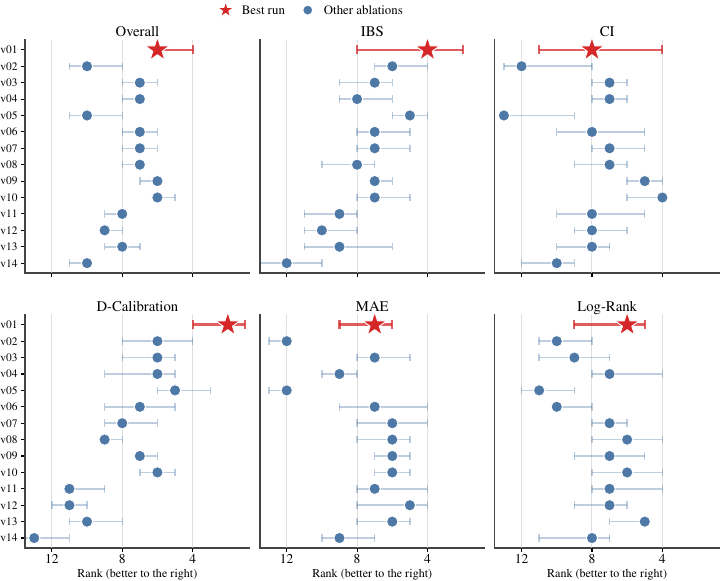}
    \caption{Ablation study over SurvivalPFN training configurations. 
    Each row corresponds to one pretrained SurvivalPFN variant, with v01 marked as the selected best validation run and used as the default checkpoint elsewhere in the paper. 
    Plotting conventions follow Figure~\ref{fig:main_results}.}
    \label{fig:ablation}
\end{figure}

\paragraph{Ablation results.}
Figure~\ref{fig:ablation} summarizes the rank of each SurvivalPFN configuration. 
The selected checkpoint, v01, achieves the strongest overall behavior:
it attains the best median rank across metrics among the evaluated configurations, and is particularly
strong on distributional metrics, ranking best on IBS and D-calibration. 
This configuration uses predictive pretraining, the survival-distribution prior, the \texttt{lognormal2normal} transformation,
the one-hot NLL objective, the Both query schedule, and a fixed train/query ratio. 
Its strong IBS and D-calibration performance suggests that this combination is especially effective for learning calibrated
posterior predictive survival distributions, which is the primary target of SurvivalPFN.

Several trends emerge from the ablation. 
First, the survival-distribution prior is consistently stronger than the naive prior under otherwise similar settings. 
This suggests that directly modeling flexible positive-time distributions provides a more useful synthetic pretraining signal than treating generic tabular outputs as raw survival times. 
The kitchen-sink prior performs competitively but does not clearly dominate the survival-distribution prior, this might indicating that simply increasing prior diversity is not sufficient; the match between the prior family and the survival-prediction target also matters.

Second, the \texttt{lognormal2normal} transformation is preferred in this set of experiments. 
The clearest comparison is between v01 and v02, replacing \texttt{lognormal2normal} with \texttt{time2quantile} substantially worsens the overall rank and degrades all five metric-specific ranks. 
This pattern suggests that the smooth positive-time coordinate and tail extrapolation provided by \texttt{lognormal2normal} are useful
for transferring across heterogeneous real-world time scales, whereas the empirical quantile transform may lose information about tail behavior.

Finally, predictive pretraining is generally helpful but not uniformly decisive in this limited ablation set. 
Similarly, varying the train/query ratio does not show a clear monotonic benefit in the evaluated subset.

\section{Concurrent Works on Tabular Foundation Models for Survival Analysis}
\label{app:concurrent_works}

We discuss two concurrent approaches that adapt TFMs to right-censored survival prediction.

\paragraph{Classification-Based Framework with Off-the-Shelf TFMs.}
\citet{kim2026tabular} propose a conversion from survival analysis to binary classification, allowing existing TFMs to be used without survival-specific pretraining. Given predefined discretization points
\[
    0 \ =\ t_0 \  < \ t_1 \ < \ \cdots \ < \ t_{K-1},
\]
they define time-indexed binary labels
\[
    Y_{i,k} \ = \ \mathbbm{1}(T_i \le t_k),
\]
so that each original tuple $(x_i,t_i,\delta_i)$ is expanded into multiple classification examples indexed by $k$. Under right censoring, labels after the censoring time are treated as missing; equivalently, their binary cross-entropy objective is evaluated only when $t_k < C_i$. Under conditional independent censoring and positivity, they show that minimizing the population binary cross-entropy loss recovers the true failure probabilities,
\[
    p(x,t_k) \ = \ \Pr(T \le t_k \mid X=x),
\]
and hence the survival probabilities $S(t_k\mid x)=1-p(x,t_k)$. 
This formulation is attractive because it can immediately use strong off-the-shelf TFMs such as MITRA, TabPFN v2.5, and TabICL v2~\citep{zhang2025mitra,grinsztajn2025tabpfn,qu2026tabiclv2}, without retraining a survival-specific model.

The main limitation is that this reduction increases the effective context size. Each subject produces up to $K-1$ time-indexed classification examples, so a dataset with $N$ subjects becomes an expanded context of order $N(K-1)$. 
This is manageable for small datasets, but can exceed the input limits of current TFMs on medium or large survival datasets, requiring subsampling and potentially discarding observed survival information. 
The method also inherits limitations of discrete-time classification, including dependence on the time grid and the need for post-hoc monotonicity correction of predicted survival curves. 
In our experiments, we include the static version of this approach as StaticSurvivalTFM.

\paragraph{Survival-Specific Prior-Fitted In-Context Learning.}
\citet{seletkov2026survival} propose Survival In-Context (SIC), a survival-specific prior-fitted model trained on synthetic right-censored datasets. 
Their data generator first samples covariates and latent risk variables $(\eta_1,\eta_2)$ from structural causal models (SCMs).
Event times are then generated using the extended-hazard model
\[
    h(t\mid x) \ = \ h_0(te^{\eta_1})e^{\eta_2},
\]
which yields
\[
    T \ = \ e^{-\eta_1} H_0^{-1}\!\left(e^{\eta_1-\eta_2}(-\log U)\right),
    \qquad U \ \sim \ \mathrm{Unif}(0,1),
\]
where $H_0^{-1}$ is chosen from a set of parametric baseline families such as Weibull, lognormal, log-logistic, Gompertz, and Birnbaum-Saunders distributions. 
Censoring is generated by random censoring assumption -- not dependent on covariates and event times.

Architecturally, SIC builds on TabICL, adds a time-event embedding, and uses a DeepHit-style discrete-time survival head~\citet{lee2018deephit} trained with a likelihood-plus-ranking loss.

SIC is closely related to SurvivalPFN in that both methods pretrain an in-context model specifically for survival prediction. 
However, the two approaches differ substantially in prior design, which is central to the PFN paradigm. 
SIC's prior is based on a parametric extended-hazard construction and only random censoring, whereas SurvivalPFN uses a broader family of identifiable right-censored DGPs, including random censoring and covariate-dependent censoring mechanisms satisfying conditional independence.
SurvivalPFN also avoids committing to explicit parametric hazard or survival families, allowing the event and censoring distributions to be generated by more flexible stochastic neural mechanisms. 
Finally, SurvivalPFN is accompanied by a posterior-predictive consistency guarantee for identifiable survival priors, whereas SIC only provides empirical evidence. 

The empirical scope also differs: SIC evaluates on a smaller benchmark with one main metric and a limited set of baselines, while our study evaluates on 61 held-out datasets, five metrics, and 21 baselines. Since SIC has not released public model weights or code, we cannot include it in our direct empirical comparison.

\end{document}